\documentclass[10pt]{article}
\usepackage[preprint]{tmlr}
\usepackage{amsmath,amsfonts,bm}

\def\eqref#1{equation~\ref{#1}}
\def\1{\bm{1}}

\DeclareMathAlphabet{\mathsfit}{\encodingdefault}{\sfdefault}{m}{sl}
\SetMathAlphabet{\mathsfit}{bold}{\encodingdefault}{\sfdefault}{bx}{n}

\usepackage{amsmath}
\usepackage{amssymb}
\usepackage{amsfonts}
\usepackage{mathtools}
\usepackage{amsthm}
\usepackage{bm}

\usepackage{graphicx}
\usepackage{subcaption}
\usepackage{float}
\usepackage{booktabs}
\usepackage{array}
\usepackage{tabularx}
\usepackage{multirow}
\usepackage{makecell}
\usepackage{longtable}
\usepackage{adjustbox}
\usepackage{siunitx}
\usepackage{placeins}

\usepackage{algorithm}
\let\TMLRauthorAND\AND
\let\AND\relax
\usepackage{algorithmic}
\let\AND\TMLRauthorAND

\usepackage{tikz}
\usetikzlibrary{arrows.meta,positioning,calc,fit,backgrounds,shapes.geometric,shadows}

\usepackage{textcomp}
\usepackage{url}
\usepackage{verbatim}
\usepackage{xcolor}
\usepackage{enumitem}
\usepackage{microtype}
\usepackage{hyperref}
\usepackage[capitalize,noabbrev]{cleveref}

\newtheorem{theorem}{Theorem}[section]
\newtheorem{proposition}[theorem]{Proposition}
\newtheorem{lemma}[theorem]{Lemma}

\theoremstyle{definition}
\newtheorem{definition}[theorem]{Definition}

\newtheorem{remark}[theorem]{Remark}

\providecommand{\exptabstretch}{1.10}
\providecommand{\exptabcolsep}{3.0pt}

\providecommand{\ind}{\mathbf{1}}

\providecommand{\proofstep}[1]{\par\smallskip\noindent\textit{#1.}\ }

\newcolumntype{C}[1]{>{\centering\arraybackslash}m{#1}}

\title{Residual-Controlled Multiplier Learning for Stochastic Constrained Decision-Making}

\author{%
\name Kang Liu \email kanyo@foxmail.com \\
\addr School of Future Technology\\
Xi'an Jiaotong University\\
Xi'an, China, 710049
\AND
\name Jianchen Hu\thanks{Corresponding author.} \email horace89@xjtu.edu.cn \\
\addr School of Automation Science and Engineering\\
Xi'an Jiaotong University\\
Xi'an, China, 710049
\AND
\name Ziyu Qu \email qzxy@cug.edu.cn \\
\addr School of Mathematics and Physics\\
China University of Geosciences\\
Wuhan, China, 430074
\AND
\name Edward Hengzhou Yan \email 23124924r@connect.polyu.hk \\
\addr Department of Industrial and Systems Engineering\\
The Hong Kong Polytechnic University\\
Hong Kong, China, 999077
\AND
\name Lun Yang \email yanglun@xjtu.edu.cn \\
\addr School of Automation Science and Engineering\\
Xi'an Jiaotong University\\
Xi'an, China, 710049
\AND
\name Meng Zhang \email mengzhang2009@xjtu.edu.cn \\
\addr School of Cyber Science and Engineering\\
Xi'an Jiaotong University\\
Xi'an, China, 710049
}

\def\month{MM}
\def\year{YYYY}
\def\openreview{\url{https://openreview.net/forum?id=XXXX}}

\begin{document}
\maketitle

\begin{abstract}
Stochastic constrained decision-making requires optimizing performance objectives while enforcing statistical requirements such as safety or fairness. However, standard primal--dual methods struggle to update multipliers robustly under stochastic mini-batch feedback, as the noise of mini-batch gradients and constraint estimates can be directly accumulated into the multiplier memory. To address this issue, we propose Residual-Controlled Multiplier Learning (RCML), which reformulates multiplier updating as projected-pressure feedback. The central idea is to decompose the projected multiplier into an effective pressure signal for primal descent and a pressure-memory residual for finite-gain multiplier tracking. To handle heterogeneous and noisy observations, we further augment this residual-integral backbone with modular stochastic stabilization components. For the convex-affine backbone, we establish finite-gain convergence, derive a stochastic residual bound under mini-batch feedback, and show that the residual feedback law admits a local KKT-residual interpretation near regular KKT points of nonconvex problems. Experiments across optimization, allocation, and fair-ranking tasks show that RCML improves feasibility control and multiplier stability while maintaining competitive objective performance. Code is released at \url{https://anonymous.4open.science/r/RCML-3114/}.
\end{abstract}

\section{Introduction}
\label{sec:introduction}

Stochastic constrained decision-making arises in many data-driven systems \cite{goh2016satisfying,chamon2020pac} where performance objectives must be optimized under statistical constraints (safety limits~\citep{tessler2019rcpo,stooke2020responsive}, fairness requirements~\citep{agarwal2018reductions,cotter2019training}, risk budgets, etc.). Primal--dual methods provide a natural algorithmic interface for solving the stochastic constrained optimization problems~\citep{xu2020primaldual,jin2022stochasticpd}, which update the primal decision variable while maintaining nonnegative multipliers that encode the pressure of constraint violations~\citep{bertsekas1982constrained,nocedal2006numerical}. Unfortunately, the multiplier update is fragile in the stochastic case because the algorithm only observes noisy mini-batch estimates of the constraints~\citep{lan2020function_constraints,zhang2021stochalm}. The major difficulty is how the stored multiplier memory should respond to noisy and varying constraint information.

Raw constraint violation signals are inherently insufficient for stochastic dual adaptation, as they fail to simultaneously capture constraint activation and multiplier dissipation. Specifically, while a signed-violation signal can actively decrease the multiplier when a constraint becomes feasible, it invariably propagates stochastic noise from inactive constraints into the dual dynamics. A positive-violation signal successfully suppresses inactive noise channels, but it loses the capacity to actively reduce the multiplier upon feasibility recovery, causing accumulated multiplier memory to become permanently stale. Although projected augmented-Lagrangian methods reinforce feasibility via a projected multiplier~\citep{rockafellar1976augmented,murtagh1982projected}, their standard full-replacement update introduces severe multiplier fluctuations when the underlying projected pressure is estimated under noisy mini-batch feedback~\citep{li2024stochasticialm,na2023stosqp}.

To address this issue, we propose Residual-Controlled Multiplier Learning (RCML), a projected-pressure feedback framework for stochastic constrained decision-making.
The central idea is to separate the projected-pressure mechanism into two coupled signals: 1) an effective pressure signal for the primal update; 2) a pressure-memory residual for multiplier adaptation.
In this way, the multiplier is no longer updated by directly accumulating raw constraint violations or by fully replacing the stored state with the projected pressure.
Instead, RCML lets the stored multiplier track the projected pressure through a finite-gain residual feedback law, which naturally supports violation activation, stale-memory release, and inactive-constraint dead-zone behaviors.

The proposed method is developed in three steps. 1) we identify the pressure-memory residual induced by the inequality augmented Lagrangian and show that its zero set exactly captures feasibility and complementarity; 2) we use this residual as the input of a stochastic primal--dual learning rule, where the projected pressure drives the primal step and the residual controls multiplier memory; 3) we equip this residual-integral backbone with stabilization modules, including constraint filtering, adaptive coordinate-wise scaling, and residual-\(\nu\)PI correction, to handle noisy and heterogeneous mini-batch constraint feedback. This control-oriented view is related to primal--dual dynamics and feedback interpretations of Lagrange multipliers~\citep{feijer2010stability,cherukuri2016asymptotic}, as well as recent I/PI multiplier updates~\citep{stooke2020responsive,sohrabi2024pi}.

We evaluate RCML on stochastic constrained optimization problems, stochastic energy-reserve allocation, nonconvex pricing-inventory allocation, and neural fair ranking, which follows exposure-based ranking fairness and cumulative-gain ranking evaluation~\citep{singh2018fairness,jarvelin2002cumulated}.
The experiments examine not only feasibility and objective quality, but also multiplier-memory behavior under noisy constraint feedback.

%%%modified
Our main contributions are as follows.
\begin{itemize}
    \item We propose a pressure--memory feedback mechanism for multiplier updates. By driving the primal step with the projected pressure and controlling the multiplier dynamics through the induced residual, the mechanism enables a single signal to naturally combine violation activation, stale-memory release, and an inactive-constraint dead-zone.

    \item We design a stochastic primal--dual algorithm that integrates finite-gain multiplier tracking with modular stabilization components. Under noisy and heterogeneous feedback, RCML achieves substantially lower multiplier fluctuation and smaller expected constraint violation than standard primal--dual updates.

    \item We establish finite-gain convergence of the deterministic convex-affine backbone and derive stochastic residual bounds for the mini-batch setting. These results reveal how the residual tracking law preserves primal--dual Lyapunov cancellation, and they quantify the roles of filtering, adaptive scaling, and dynamic correction inside the projected-pressure feedback loop.
\end{itemize}

\section{Related Work}
\label{sec:related_work}

This section reviews prior work along the three ingredients that motivate RCML.
First, constrained learning explains why multiplier states are needed: statistical requirements such as safety and fairness cannot usually be enforced by a fixed penalty weight, so the tradeoff between objective descent and constraint satisfaction must be adjusted dynamically.
Second, stochastic constrained optimization explains why this adjustment is difficult: mini-batch constraint feedback is noisy, and therefore multiplier memory can become unstable when it is driven by raw violation signals or direct projected replacement.
Third, feedback-control interpretations of multiplier dynamics explain where RCML intervenes: once the multiplier is viewed as a memory state in a feedback loop, the central design question becomes what signal should drive this memory.

\paragraph{Constrained learning and dynamic multiplier states}
Many data-driven decision systems must optimize a prediction or control objective while satisfying statistical requirements.
Safety and reliability requirements are often formulated as expectation constraints on learned policies or decisions~\citep{goh2016satisfying,chamon2020pac}, while fairness-constrained learning imposes group-level constraints during classification or ranking~\citep{agarwal2018reductions,cotter2019nondiff,singh2018fairness}.
These constraints depend on the current model, data distribution, and sampling process, so they cannot always be replaced by a fixed regularization weight.
Constrained learning methods therefore introduce dual variables, proxy losses, or game-theoretic players to adapt the tradeoff between objective descent and constraint satisfaction during training~\citep{cotter2019training,agarwal2018reductions,chamon2020pac}.
Related constrained reinforcement learning methods use Lagrangian or primal--dual mechanisms to control expected cumulative constraint violations during policy optimization~\citep{tessler2019rcpo,stooke2020responsive}.
In these methods, the multiplier is not only a static penalty coefficient, but also a dynamic memory state that records past constraint pressure and affects future primal updates.
This makes multiplier-memory updating a central algorithmic component, especially when the constraint signal is observed only through stochastic mini-batches.

\paragraph{Stochastic constrained optimization under noisy feedback}
Stochastic constrained optimization studies how to update primal and dual variables using sampled objective and constraint information.
Stochastic approximation methods provide foundational mechanisms for handling expectation constraints from noisy samples~\citep{wang2016stochastic,lan2020function_constraints}.
Stochastic primal--dual methods extend this idea by coupling primal descent with multiplier ascent or correction under sampled feedback~\citep{xu2020primaldual,jin2022stochasticpd}.
When constraints are nonlinear or functional, constraint extrapolation and linearization-based methods improve feasibility control without requiring exact full-batch projections~\citep{boob2023stochastic}.
For nonconvex constrained problems, single-loop perturbed ascent and related first-order schemes reduce the need for nested constrained subproblem solves~\citep{lu2022gdpa,alacaoglu2024complexity}.
Another line of work replaces hard constraint handling with smoothed penalties or relaxed feasibility surrogates to obtain tractable stochastic updates~\citep{huang2025smoothed,yang2026singleloop}.
Stochastic augmented Lagrangian methods combine multiplier updates with penalty-based correction to stabilize expectation-constrained optimization under sampling noise~\citep{zhang2021stochalm,li2024stochasticialm}.
Stochastic SQP methods instead build local quadratic models and use active-set or proximal corrections to obtain more structured constrained steps~\citep{na2023stosqp,cui2025twophase}.
These approaches improve feasibility control by changing the primal--dual update, the local model, or the correction mechanism.
However, the multiplier memory is still commonly driven by 1) raw signed violations, 2) positive violations, or 3) direct projected multiplier replacement.
Thus, existing stochastic constrained optimization mainly stabilizes the optimization loop around the constraint signal, whereas the design of the multiplier-memory input itself remains less explicit.

% modified the first sentence
\paragraph{Feedback-control views and multiplier signal design}
To systematically re-engineer this memory input, a control-theoretic perspective proves essential, interpreting multiplier updates directly as closed-loop feedback laws.
Classical augmented Lagrangian and projected multiplier methods already use the projected quantity to handle inequality constraints~\citep{rockafellar1976augmented,murtagh1982projected}.
Standard nonlinear programming treatments further connect projected or augmented multiplier updates with feasibility and complementarity enforcement~\citep{bertsekas1982constrained,nocedal2006numerical}.
Dynamical analyses of primal--dual algorithms show that multiplier evolution can be understood as a feedback process coupled with primal motion~\citep{feijer2010stability,cherukuri2016asymptotic}.
In safe reinforcement learning, the same idea appears as integral multiplier control, where accumulated constraint errors regulate policy updates~\citep{tessler2019rcpo}.
PID-style Lagrangian methods make this control interpretation more explicit by adding proportional or derivative components to improve transient constraint behavior~\citep{stooke2020responsive,sohrabi2024pi}.
Recent continuous-time and nonsmooth analyses further formalize the feedback nature of Lagrangian dynamics~\citep{cerone2025framework,cerone2026nonsmooth}.
These works clarify that the multiplier is a feedback state, but the feedback input is usually built from raw constraint errors or their filtered variants.
Our proposed RCML follows the same control viewpoint but changes the input to the multiplier-memory loop.
Instead of feeding the memory update directly with raw violations or fully replacing the memory by the projected multiplier, RCML first forms the projected effective multiplier and then updates memory through the induced pressure-memory residual.
This residual input combines violation activation, stale-memory release, and inactive-constraint dead-zone behavior within one projected-pressure feedback interface.

\section{Problem Statement and Preliminaries}
\label{sec:problem_preliminaries}

\subsection{Stochastic Constrained Decision-Making Problem}
\label{subsec:problem_interface}

We consider the stochastic inequality-constrained problem
\citep{lan2020function_constraints,zhang2021stochalm}
\begin{equation}
	\label{eq:stoch_problem}
	\min_{x\in\mathcal X}\; f(x):=\mathbb E_{\xi}[\ell(x;\xi)]
	~~
	\mathrm{s.t.}~~
	c_i(x):=\mathbb E_{\xi}[h_i(x;\xi)]\le 0,~~ i=1,\ldots,m ,
\end{equation}
where \(x\in\mathcal X\subseteq\mathbb R^d\) is the decision variable,
\(\xi\) denotes a random sample,
\(\ell(x;\xi)\) is the sample-wise objective loss, and \(h_i(x;\xi)\) is the
sample-wise contribution to the \(i\)-th constraint. We write
\(c(x)=[c_1(x),\ldots,c_m(x)]^\top\),
\(h(x;\xi)=[h_1(x;\xi),\ldots,h_m(x;\xi)]^\top\), and denote the
constraint Jacobian by \(J_c(x)\in\mathbb R^{m\times d}\). In learning
applications, the mini-batch constraint Jacobian \(\widehat J_{c,k}\) is obtained
by differentiating the mini-batch constraint estimate with respect to the model
parameters.

The expectation-form Lagrangian of \eqref{eq:stoch_problem} is
\citep{bertsekas1982constrained,nocedal2006numerical}
\begin{equation}
	\label{eq:stoch_lagrangian}
	\mathcal L(x,u)
	=
	f(x)+u^\top c(x)
	=
	\mathbb E_{\xi}\bigl[\ell(x;\xi)+u^\top h(x;\xi)\bigr],
	~~
	u\in\mathbb R_+^m .
\end{equation}
Its primal descent direction is
$\nabla_x\mathcal L(x,u)=\nabla f(x)+J_c(x)^\top u$, and its multiplier-side ascent direction is \(c(x)\). Under exact information, a
projected primal--dual method updates the primal variable along
\(\nabla f(x_k)+J_c(x_k)^\top u_k\) and updates the multiplier along
\(c(x_k)\), with projection onto \(\mathcal X\) and \(\mathbb R_+^m\).

In stochastic learning or data-driven decision-making, the exact quantities
\(\nabla f(x_k)\), \(c(x_k)\), and \(J_c(x_k)\) are usually unavailable at each
iteration. The algorithm instead observes mini-batch quantities
\(\hat g_k\), \(\hat c_k\), and \(\widehat J_{c,k}\), which are used as
stochastic estimates of \(\nabla f(x_k)\), \(c(x_k)\), and \(J_c(x_k)\),
respectively. The key issue is not only that these estimates are noisy, but also
that the noise can be transmitted into the stored multiplier memory if the memory
is updated directly from raw constraint feedback.

\subsection{A Generalized Primal--Dual Interface}
\label{subsec:generalized_interface}

To cover both classical stochastic primal--dual updates and the proposed RCML
update, we use the following generalized interface:
\begin{equation}
	\label{eq:pd_template}
	\begin{aligned}
	x_{k+1}
	=
	\Pi_{\mathcal X}
	\bigl(
	x_k-\alpha_k(\hat g_k+\widehat J_{c,k}^{\top}\mu_k)
	\bigr),~~
	u_{k+1}
	=
	\Pi_{\mathbb R_+^m}(u_k+\eta_k s_k).
	\end{aligned}
\end{equation}
Here \(\Pi_{\mathcal S}\) denotes Euclidean projection onto a set
\(\mathcal S\), \(u_k\in\mathbb R_+^m\) is the stored multiplier memory,
\(\mu_k\in\mathbb R_+^m\) is the effective multiplier pressure applied to the
primal update, and \(s_k\in\mathbb R^m\) is the feedback signal used to update
the multiplier memory. The stepsizes \(\alpha_k>0\) and \(\eta_k>0\) control the
primal and multiplier-memory updates.

The classical stochastic primal--dual update is recovered from
\eqref{eq:pd_template} by choosing $\mu_k=u_k, s_k=\hat c_k$, or by replacing \(s_k\) with the positive violation signal \([\hat c_k]_+\)~\citep{bertsekas1982constrained}.
This choice ties two different roles to the same raw constraint-feedback loop: the stored multiplier \(u_k\) is used as the pressure in the primal update, and the mini-batch constraint estimate is directly integrated into the multiplier memory. Consequently, stochastic fluctuations in \(\hat c_k\) can accumulate in \(u_k\), especially when the constraint signal is highly variable across mini-batches.

The central design question is therefore whether the pressure applied to the primal update and the signal used to update multiplier memory should be tied to the same feedback signal. RCML answers this question by separating these two roles. It constructs a projected effective pressure for the primal update and a pressure-memory residual for updating the stored multiplier memory.

\subsection{Projected Pressure and Memory Residual from the Augmented Lagrangian}
\label{subsec:lagrangian_alm}

We next introduce the projected-pressure signal induced by the inequality
augmented Lagrangian. Let
\(\boldsymbol\rho=(\rho_1,\ldots,\rho_m)^\top\in\mathbb R_{++}^m\) be a
coordinate-wise pressure scale. The inequality augmented Lagrangian is
\citep{rockafellar1976augmented,bertsekas1982constrained}
\begin{equation}
	\label{eq:ineq_aug_lag}
	\mathcal L_{\boldsymbol\rho}(x,u)
	=
	f(x)
	+
	\frac12\sum_{i=1}^m
	\frac{[u_i+\rho_i c_i(x)]_+^2-u_i^2}{\rho_i},
\end{equation}
where \([\,\cdot\,]_+\) denotes coordinate-wise projection onto the nonnegative
orthant. This augmented-Lagrangian object induces two signals,
\begin{equation}
	\label{eq:pressure_residual}
	\lambda_{\boldsymbol\rho}(x,u)
	:=
	[u+\boldsymbol\rho\odot c(x)]_+,
	~~
	d_{\boldsymbol\rho}(x,u)
	:=
	\lambda_{\boldsymbol\rho}(x,u)-u ,
\end{equation}
where \(\odot\) denotes the Hadamard product. The vector
\(\lambda_{\boldsymbol\rho}(x,u)\) is the projected effective multiplier
pressure, while \(d_{\boldsymbol\rho}(x,u)\) is the pressure-memory residual that
measures the gap between the projected pressure and the stored multiplier
memory.
Direct differentiation of \eqref{eq:ineq_aug_lag} obtains 
$\nabla_x\mathcal L_{\boldsymbol\rho}(x,u) = \nabla f(x)+J_c(x)^\top\lambda_{\boldsymbol\rho}(x,u),~\nabla_u\mathcal L_{\boldsymbol\rho}(x,u) = \boldsymbol\rho^{-1}\odot d_{\boldsymbol\rho}(x,u),$
where
\(\boldsymbol\rho^{-1}=(1/\rho_1,\ldots,1/\rho_m)^\top\). Therefore,
\(\lambda_{\boldsymbol\rho}\) is the multiplier pressure appearing in the primal derivative, and \(d_{\boldsymbol\rho}\) is the multiplier-side residual generated by the same augmented-Lagrangian formulation.

In the isotropic case \(\boldsymbol\rho=\rho_0\mathbf 1\), we write $\lambda_{\rho_0}:=\lambda_{\rho_0\mathbf 1},~d_{\rho_0}:=d_{\rho_0\mathbf 1}$.
Then $\nabla_u\mathcal L_{\rho_0\mathbf 1}(x,u)=\rho_0^{-1}d_{\rho_0}(x,u)$. Classical projected augmented-Lagrangian methods use the projected pressure as the next multiplier memory. In an online update, this gives $u_{k+1}=\lambda_{\boldsymbol\rho}(x_k,u_k),$ whereas an offline augmented-Lagrangian convention may use $u_{k+1}=\lambda_{\boldsymbol\rho}(x_{k+1},u_k).$ Thus, the memory is fully replaced by the current projected pressure. This unit-gain replacement is effective under exact deterministic information, but under stochastic mini-batch feedback it can transfer the full fluctuation of the projected pressure estimate into the multiplier memory.

\subsection{Feedback-Control Interpretation}
\label{subsec:control_interpretation}

The generalized interface \eqref{eq:pd_template} can be interpreted as a
closed-loop feedback system. The primal state \(x_k\) is regulated by the effective pressure \(\mu_k\), and the stored multiplier memory \(u_k\) is updated through the feedback signal \(s_k\).

For the classical stochastic primal--dual update with \(s_k=\hat c_k\), ignoring the nonnegative projection gives the formal recursion $u_k=u_0+\sum_{t=0}^{k-1}\eta_t\hat c_t$. Thus, the multiplier memory behaves as an integral controller that accumulates the measured constraint signal. This interpretation explains its sensitivity to high-frequency stochastic noise: mini-batch fluctuations in \(\hat c_t\) are directly accumulated into \(u_k\). A signed violation signal can release multiplier mass when a constraint becomes feasible, but it also transmits fluctuations from inactive constraints. A positive violation signal suppresses negative inactive fluctuations, but it cannot actively remove stale positive multiplier memory through the constraint signal itself.

The projected augmented-Lagrangian pressure in \eqref{eq:pressure_residual} provides a richer feedback signal. In the active region \(u_i+\rho_i c_i(x)>0\), the \(i\)-th projected pressure satisfies $	\lambda_{\boldsymbol\rho,i}(x,u)=u_i+\rho_i c_i(x).$ Hence, the pressure combines an integral component \(u_i\) with a proportional correction \(\rho_i c_i(x)\). In the inactive region \(u_i+\rho_i c_i(x)\le0\), the projection sets $\lambda_{\boldsymbol\rho,i}(x,u)=0,$ introducing the saturation and dead-zone behavior required for inequality constraints.

From this viewpoint, projected augmented-Lagrangian methods already provide a useful projected pressure for the primal update, but they couple this pressure with unit-gain memory replacement. RCML keeps the projected-pressure correction for the primal update and replaces unit-gain memory replacement with finite-gain residual tracking. The projected pressure \(\lambda_{\boldsymbol\rho}\) is used as the effective pressure in the primal channel, while the induced residual
\(d_{\boldsymbol\rho}\) becomes the feedback signal that controls the multiplier memory.

\section{Methodology}
\label{sec:methodology}
Building on the two signals identified in the previous section, we now construct a systematic framework that uses $\lambda$ for the primal update and $d$ for memory tracking.  We first analyze the residual signal structure, then embed it into a unified algorithm with finite-gain memory control, and finally introduce optional stabilization modules for stochastic feedback.

\subsection{Residual-Controlled Multiplier Signal}
\label{subsec:residual_signal_design}

The key signal in RCML is the pressure-memory residual
\(d_{\boldsymbol\rho}\) defined in \eqref{eq:pressure_residual}, which measures how far the stored multiplier memory \(u\) is from the projected pressure
\(\lambda_{\boldsymbol\rho}(x,u)\).  From~\eqref{eq:pressure_residual}, the coordinate-wise form is
\[
d_{\boldsymbol\rho,i}(x,u)
=
\begin{cases}
\rho_i c_i(x), & u_i+\rho_i c_i(x)>0,\\[2pt]
-u_i,          & u_i+\rho_i c_i(x)\le0 .
\end{cases}
\]
This piecewise structure gives the residual three feedback behaviors.  First, it provides \emph{violation activation}: when constraint \(i\) is violated and the projected pressure is active, \(d_{\boldsymbol\rho,i}(x,u)=\rho_i c_i(x)>0\), which increases the multiplier memory in proportion to the violation.  Second, it provides \emph{stale-memory release}: if a previously active constraint becomes sufficiently feasible while \(u_i>0\), then \(u_i+\rho_i c_i(x)\le0\) and \(d_{\boldsymbol\rho,i}(x,u)=-u_i\), so the obsolete multiplier value is actively removed.  Third, it provides an \emph{inactive dead-zone}: if \(u_i=0\) and \(c_i(x)<0\), then \(d_{\boldsymbol\rho,i}(x,u)=0\), so a strictly inactive constraint does not inject negative noise into the multiplier dynamics.

This explains why \(d_{\boldsymbol\rho}\) is more informative than either the raw violation \(c(x)\) or its positive part \([c(x)]_+\).  The signed signal \(c(x)\) can release memory but also transmits inactive negative fluctuations; the positive signal \([c(x)]_+\) suppresses negative inactive fluctuations but cannot release stale positive memory.  The residual \(d_{\boldsymbol\rho}\) unifies activation, release, and dead-zone behavior in a single projected-pressure feedback signal.

The residual also carries an exact optimality meaning.  As shown in Proposition~\ref{prop:residual_zero_set} (Section~\ref{sec:theory}), \(d_{\boldsymbol\rho}(x,u)=0\) is equivalent to \(u\ge0\), \(c(x)\le0\), and \(u_i c_i(x)=0\) for all \(i\).  Together with projected stationarity of the primal step, this yields the KKT system.  Hence, RCML employs a memory-control signal that is both operationally meaningful in the feedback loop and mathematically aligned with constrained optimality.

\subsection{RCML Finite-Gain Memory Tracking and Unified Algorithm}
\label{subsec:finite_gain_tracking}

RCML applies the projected pressure to the primal update but adjusts the multiplier memory through the residual.  At iteration \(k\), the algorithm first forms a constraint signal \(\tilde c_k\), which may be the raw mini-batch estimate \(\hat c_k\) or a filtered version introduced in Section~\ref{subsec:stochastic_components}.  It also selects a pressure-scale vector \(\boldsymbol\rho_k\in\mathbb R_{++}^m\), which can be fixed, coordinate-wise, or adaptively updated.

Using the current signal \(\tilde c_k\) and scale \(\boldsymbol\rho_k\), the online projected pressure and residual are obtained by analogy with
\eqref{eq:pressure_residual}:
\[
\lambda_k=[u_k+\boldsymbol\rho_k\odot\tilde c_k]_+,~~
d_k=\lambda_k-u_k .
\]
In the generalized primal--dual template \eqref{eq:pd_template}, RCML sets the primal pressure to \(\mu_k=\lambda_k\) and constructs the memory feedback signal as \(s_k=\kappa_{\text{I}} d_k\) with gain \(\kappa_{\text{I}}>0\).  The multiplier memory update then becomes
\begin{equation}
\label{eq:rcml_memory_update}
u_{k+1}
=
\Pi_{\mathbb R_+^m}\bigl(u_k+\eta_k\kappa_{\text{I}} d_k\bigr).
\end{equation}
Let \(\beta_k:=\eta_k\kappa_{\text{I}}\).  If \(0\le\beta_k\le1\), then using \(d_k=\lambda_k-u_k\) gives \(u_{k+1}=(1-\beta_k)u_k+\beta_k\lambda_k\), so the projection in \eqref{eq:rcml_memory_update} can be dropped because both \(u_k\) and \(\lambda_k\) are nonnegative.  Thus the multiplier memory moves only a fraction \(\beta_k\) toward the current projected pressure.

The classical projected ALM update is recovered as the unit-gain case
\(\beta_k=1\).  Smaller values of \(\beta_k\) produce a smoother multiplier
trajectory and reduce sensitivity to instantaneous mini-batch noise.  In this way,
RCML retains the projected-pressure correction of ALM while avoiding its
aggressive full-replacement memory update.

Algorithm~\ref{alg:rcml_unified} summarizes the unified RCML execution flow.
The minimal RCML backbone consists of two choices: \(\mu_k=\lambda_k\) and \(s_k=\kappa_{\text{I}} d_k\). All other components, including constraint filtering,
adaptive scaling, and residual-\(\nu\)PI correction, are optional stochastic
stabilization modules.

%% modified
\begin{algorithm}[ht]
\caption{Unified Residual-Controlled Multiplier Learning}
\label{alg:rcml_unified}
\let\AND\relax
\begin{algorithmic}[1]
\REQUIRE Choose stepsizes $\{\alpha_k, \eta_k\}>0$, initialize $x_0\in\mathcal X$ and $u_0\in\mathbb R_+^m$. If filtering, adaptive scaling, or residual-$\nu$PI correction is used, initialize their auxiliary states $\bar c_{-1}=0$, $v_0=0$, $\xi_{-1}=0$ and choose memory signal \(s_k\) (see Table~\ref{tab:rcml_variants}).
\FOR{$k=0,1,\ldots,T-1$}
    \STATE Observe mini-batch estimates $\hat g_k$, $\hat c_k$, and $\widehat J_{c,k}$.
    \STATE Form the constraint signal $\tilde c_k$ from $\hat c_k$ (e.g., via filtering or using the raw estimate).
    \STATE Compute the projected pressure and residual:
          $\lambda_k=[u_k+\boldsymbol\rho_k\odot\tilde c_k]_+$,\;
          $d_k=\lambda_k-u_k$.
    \STATE Update the primal variable:
          $x_{k+1}=\Pi_{\mathcal X}\bigl(x_k-\alpha_k[\hat g_k+\widehat J_{c,k}^{\top}\lambda_k]\bigr)$.
    \STATE Update the multiplier memory:
          $u_{k+1} = \Pi_{\mathbb R_+^m}(u_k+\eta_k s_k)$.
\ENDFOR
\end{algorithmic}
\end{algorithm}

The algorithm contains three signal channels:
1) \emph{Measurement channel}: converts the raw mini-batch constraint estimate
\(\hat c_k\) into the processed signal \(\tilde c_k\) (the filtered option is
detailed in Section~\ref{subsec:stochastic_components});
2) \emph{Pressure channel}: computes the projected pressure \(\lambda_k\), which
is then applied to the primal step through \(\widehat J_{c,k}^{\top}\lambda_k\);
3) \emph{Memory-feedback channel}: forms the residual-based signal \(s_k\) and
updates the stored multiplier memory \(u_k\).

\begin{table*}[ht]
\centering
\caption{RCML variants and classical baselines under the unified projected-pressure
execution flow. $ \lambda_k $ and $ d_k $ denote the projected pressure and
pressure-memory residual obtained from the selected constraint signal
$ \tilde c_k $ and pressure-scale vector $ \boldsymbol\rho_k $. Setting
$ \eta_k=1 $ and $ s_k=d_k $ recovers the classical projected
augmented-Lagrangian replacement update~\citep{rockafellar1976augmented}.}
\label{tab:rcml_variants}
\footnotesize
\setlength{\tabcolsep}{2.4pt}
\renewcommand{\arraystretch}{1.22}
\begin{tabular}{@{}
C{0.145\textwidth}
C{0.120\textwidth}
C{0.170\textwidth}
C{0.145\textwidth}
C{0.205\textwidth}
C{0.120\textwidth}
@{}}
\toprule
\textbf{Method}
&
\textbf{Constraint signal}
&
\textbf{Pressure scale}
&
\textbf{Memory signal}
&
\textbf{Memory update}
&
\textbf{Main role}
\\
\midrule

\makecell{\textsc{SGDA-Signed} \\ \citep{jin2022stochasticpd}}
&
$ \tilde c_k=\hat c_k $
&
Not used
&
$ s_k=\hat c_k $
&
$ \Pi_{\mathbb R_+^m}(u_k+\eta_k s_k) $
&
Signed integral
\\
\addlinespace[1pt]

\makecell{\textsc{SGDA-Positive} \\ \citep{xu2020primaldual}}
&
$ \tilde c_k=\hat c_k $
&
Not used
&
$ s_k=[\hat c_k]_+ $
&
$ \Pi_{\mathbb R_+^m}(u_k+\eta_k s_k) $
&
Positive integral
\\
\addlinespace[1pt]

\makecell{\textsc{Projected-ALM} \\ \citep{bertsekas1982constrained}}
&
$ \tilde c_k=\hat c_k $
&
$ \boldsymbol\rho_k=\rho_0\mathbf 1 $
&
$ s_k=d_k $
&
\makecell{$ \eta_k=1 $\\ $ u_{k+1}=\lambda_k $}
&
Unit-gain tracking
\\
\addlinespace[1pt]

\textsc{Residual-I}
&
$ \tilde c_k=\hat c_k $
&
$ \boldsymbol\rho_k=\rho_0\mathbf 1 $
&
$ s_k=\kappa_{\text{I}} d_k $
&
\makecell{$ u_{k+1}=u_k+\eta_k s_k $\\ $ 0\le\eta_k\kappa_{\text{I}}\le1 $}
&
Finite-gain tracking
\\
\addlinespace[1pt]

\textsc{RCML-Core}
&
$ \tilde c_k=\bar c_k $
&
$ \boldsymbol\rho_k=\rho_0\mathbf 1 $
&
$ s_k=\kappa_{\text{I}} d_k $
&
\makecell{$ u_{k+1}=u_k+\eta_k s_k $\\ $ 0\le\eta_k\kappa_{\text{I}}\le1 $}
&
Filtered tracking
\\
\addlinespace[1pt]

\textsc{RCML-Adaptive}
&
$ \tilde c_k=\bar c_k $
&
\makecell{$ \boldsymbol\rho_k= $\\$ [\rho_{k,1},\ldots,\rho_{k,m}]^\top $}
&
$ s_k=\kappa_{\text{I}} d_k $
&
\makecell{$ u_{k+1}=u_k+\eta_k s_k $\\ $ 0\le\eta_k\kappa_{\text{I}}\le1 $}
&
Adaptive scaling
\\
\addlinespace[1pt]

\textsc{RCML-Robust}
&
$ \tilde c_k=\bar c_k $
&
\makecell{$ \boldsymbol\rho_k= $\\$ [\rho_{k,1},\ldots,\rho_{k,m}]^\top $}
&
\makecell{$ s_k=\kappa_{\text{I}} d_k $\\ $ +\kappa_{\text{P}}(\xi_k-\xi_{k-1}) $}
&
$ \Pi_{\mathbb R_+^m}(u_k+\eta_k s_k) $
&
Dynamic correction
\\

\bottomrule
\end{tabular}

\vspace{1ex}
\begin{minipage}{0.98\textwidth}
\footnotesize
\emph{Notes.} The filtered signal $ \bar c_k $, adaptive coordinates
$ \rho_{k,i} $, and residual smoothing state $\xi_k$ are defined in
Section~\ref{subsec:stochastic_components}. For residual-integral variants, the
projection in the multiplier update is inactive when $0\le\eta_k\kappa_{\text{I}}\le1$
because $u_{k+1}$ becomes a convex combination of $u_k$ and $\lambda_k$.
For \textsc{RCML-Robust}, the additional dynamic correction may break this convex
combination, so the non-negative projection is retained as a safeguard.
\end{minipage}
\end{table*}
\subsection{Stochastic Stabilization Modules}
\label{subsec:stochastic_components}

In stochastic constrained learning, mini-batch feedback can be very noisy.  We
therefore equip RCML with three optional stabilization modules.  These modules do
not alter the meaning of the projected-pressure residual; they only modify the
measurement, pressure, or memory-feedback channel.

\paragraph{Constraint filtering}
The first module acts on the measurement channel.  Instead of forming the pressure
directly from the raw mini-batch estimate \(\hat c_k\), RCML can use an
exponential moving average:
\begin{equation}
\label{eq:constraint_filter}
\tilde c_k = (1-\gamma_k)\,\tilde c_{k-1} + \gamma_k\,\hat c_k,
~~ \gamma_k\in(0,1].
\end{equation}
The parameter \(\gamma_k\) controls the noise--lag trade-off.  When \(\gamma_k=1\),
the module is inactive and \(\tilde c_k=\hat c_k\).  Smaller values reduce
high-frequency stochastic fluctuations but introduce delay in detecting rapid
constraint changes.

\paragraph{Adaptive coordinate-wise pressure scaling}
The second module acts on the pressure channel.  A single isotropic pressure scale
\(\rho_0\mathbf 1\) can be unsuitable when constraints have very different
magnitudes: a large-scale constraint may dominate the projected pressure, while a
small-scale one may receive insufficient correction.  To reduce this imbalance,
RCML can adapt the coordinate-wise scale
\(\boldsymbol\rho_k=[\rho_{k,1},\ldots,\rho_{k,m}]^\top\).

Using the filtered constraint signal, we maintain a bias-corrected second-moment
estimate:
\[
v_{k+1}
= (1-\eta_v)v_k + \eta_v\,\tilde c_k^{\,2},
~~
\widehat v_{k+1}
= \frac{v_{k+1}}{1-(1-\eta_v)^{k+1}} .
\]
The adaptive coordinate-wise pressure scale is then
\begin{equation}
\label{eq:adaptive_scale}
\rho_{k,i}
=
\operatorname{clip}\!
\left(
\frac{\kappa_\rho}{\sqrt{\widehat v_{k+1,i}+\epsilon}},\;
\rho_{\min},\;
\rho_{\max}
\right),
~~ i=1,\ldots,m .
\end{equation}
Here \(\kappa_\rho>0\) is the base scale, \(\epsilon>0\) prevents division by
zero, and \(\rho_{\min},\rho_{\max}\) keep the pressure scale within a stable
range.  The bias correction in \(\widehat v_{k+1}\) is exact when the second moment
of the constraint signal is stationary; under distribution shift or rapidly moving
iterates it should be viewed as a pragmatic re-scaling rather than an exact
moment correction.

\paragraph{Residual-\(\nu\)PI dynamic correction}
The third module acts on the memory-feedback channel.  The pure residual-integral
update \(s_k=\kappa_{\text{I}} d_k\) gives a conservative finite-gain controller.  To
accelerate the transient response when constraint pressure changes rapidly, RCML
can add a dynamic correction based on a smoothed residual state:
\begin{equation}
\label{eq:residual_filter}
\xi_k = \nu\,\xi_{k-1} + (1-\nu)\,d_k,
~~ \nu\in[0,1).
\end{equation}
The difference \(\xi_k-\xi_{k-1}\) captures the recent trend of the
pressure-memory mismatch.  The complete residual-\(\nu\)PI memory signal is
\begin{equation}
\label{eq:nupi_update}
s_k
= \kappa_{\text{I}} d_k + \kappa_{\text{P}}(\xi_k-\xi_{k-1}),
~~ \kappa_{\text{I}}>0,\;\; \kappa_{\text{P}}\ge0.
\end{equation}
The term \(\kappa_{\text{I}} d_k\) provides integral action on the mismatch, while
\(\kappa_{\text{P}}(\xi_k-\xi_{k-1})\) adds a proportional correction on its filtered
change.  Setting \(\kappa_{\text{P}}=0\) recovers the pure finite-gain residual-integral
backbone.  Because the dynamic term may cause the memory update to deviate from a
convex combination of \(u_k\) and \(\lambda_k\), the projection onto
\(\mathbb R_+^m\) is retained for this variant.

\section{Theoretical Analysis}
\label{sec:theory}

The results are organized as follows.
Section~\ref{subsec:theory_residual_geometry} shows that
\(d_{\boldsymbol\rho}\) is an exact feasibility--complementarity residual and
that, together with projected stationarity, it characterizes the KKT system.
Section~\ref{subsec:theory_finite_gain_backbone} analyzes the deterministic
convex-affine residual-integral backbone corresponding to the finite-gain update
in Section~\ref{subsec:finite_gain_tracking}.
Section~\ref{subsec:theory_stochastic_bound} derives a stopped finite-time
stochastic residual bound under mini-batch feedback.
Sections~\ref{subsec:theory_filtering} and
\ref{subsec:theory_scaling_dynamic} then explain how constraint filtering,
adaptive pressure scaling, and residual-\(\nu\)PI correction enter the
measurement, pressure, and memory-feedback channels.

\paragraph{Scope} The residual identities in Section~\ref{subsec:theory_residual_geometry} hold
for general smooth nonlinear constraints \(c(x)\). The dynamical convergence and
stochastic residual bounds in Sections~\ref{subsec:theory_finite_gain_backbone}
and \ref{subsec:theory_stochastic_bound} are established for the convex-affine
backbone \(c(x)=Ax-b\) and \(\mathcal X=\mathbb R^d\). In these two
subsections, the vector scale is specialized to the isotropic form
\(\boldsymbol\rho=\rho_0\mathbf 1\) with \(\rho_0>0\), and the scalar symbol
\(\rho_0\) is used only for this isotropic specialization. The nonlinear formulation in
Section~\ref{sec:methodology} defines the implementable RCML mechanism, while the
convex-affine analysis isolates the finite-gain residual-memory dynamics in a setting where the Lyapunov cancellation can be proved exactly. The final
nonconvex discussion is used to interpret the residual diagnostics near
regular KKT points.

\begin{definition}\label{denf1}
For \(x\in\mathcal X\), the projected stationarity residual is
\[
\mathcal G_{\alpha,\boldsymbol\rho}^{\mathcal X}(x,u)
:=
\alpha^{-1}
\left\|
x-\Pi_{\mathcal X}
\bigl(
x-\alpha[\nabla f(x)+J_c(x)^\top\lambda_{\boldsymbol\rho}(x,u)]
\bigr)
\right\|,
\]
and the combined residual is
\[
\mathcal R_{\alpha,\boldsymbol\rho}^{\mathcal X}(x,u)
:=
\bigl(\mathcal G_{\alpha,\boldsymbol\rho}^{\mathcal X}(x,u)\bigr)^2
+
\|d_{\boldsymbol\rho}(x,u)\|^2 .
\]
\end{definition}

The projected stationarity residual follows the standard gradient-mapping
and normal-cone characterization of constrained stationarity
\citep{bertsekas1982constrained,nocedal2006numerical,bauschke2017convex}. When \(\mathcal X=\mathbb R^d\), the projected stationarity term reduces to the gradient norm:
\(\mathcal R_{\alpha,\boldsymbol\rho}^{\mathbb R^d}(x,u)
=
\|\nabla f(x)+J_c(x)^\top\lambda_{\boldsymbol\rho}(x,u)\|^2
+
\|d_{\boldsymbol\rho}(x,u)\|^2\).

\subsection{Residual Geometry and KKT Characterization}
\label{subsec:theory_residual_geometry}

We first prove that the pressure-memory residual \(d_{\boldsymbol\rho}\) defined
in \eqref{eq:pressure_residual} is an exact optimality residual. Throughout this
subsection, \(\boldsymbol\rho\in\mathbb R_{++}^m\) is allowed to be
coordinate-wise.

\begin{proposition}[Residual zero set]
\label{prop:residual_zero_set}
For any \(\boldsymbol\rho\in\mathbb R_{++}^m\),
\[
d_{\boldsymbol\rho}(x,u)=0
~~\Longleftrightarrow~~
u\ge0,~~ c(x)\le0,~~ u_i c_i(x)=0,\;\; i=1,\ldots,m .
\]
\end{proposition}

\begin{proof}
\noindent
\textit{(i) Necessity.}
Let \(d_{\boldsymbol\rho}(x,u)=0\). From \eqref{eq:pressure_residual}, this is
equivalent to \(u=[u+\boldsymbol\rho\odot c(x)]_+\). Since \([\cdot]_+\) is the
projection onto \(\mathbb R_+^m\), for each coordinate \(i\):
if \(u_i>0\), then \(u_i=[u_i+\rho_i c_i(x)]_+\) implies \(u_i+\rho_i c_i(x)>0\),
hence \(u_i=u_i+\rho_i c_i(x)\) and \(c_i(x)=0\);
if \(u_i=0\), then \(0=[\rho_i c_i(x)]_+\) gives \(c_i(x)\le0\).
Therefore \(c(x)\le0\) and \(u_i c_i(x)=0\) for all \(i\).

\noindent
\textit{(ii) Sufficiency.}
Let \(u\ge0\), \(c(x)\le0\), and \(u_i c_i(x)=0\).
If \(u_i>0\), complementarity implies \(c_i(x)=0\), hence
\([u_i+\rho_i c_i(x)]_+=u_i\).
If \(u_i=0\), feasibility implies \(c_i(x)\le0\), hence
\([u_i+\rho_i c_i(x)]_+=[\rho_i c_i(x)]_+=0=u_i\).
Thus \([u+\boldsymbol\rho\odot c(x)]_+=u\), i.e., \(d_{\boldsymbol\rho}(x,u)=0\).
\end{proof}

\begin{proposition}[KKT characterization by the combined residual]
\label{prop:kkt_characterization}
For any \(\alpha>0\),
\(\mathcal R_{\alpha,\boldsymbol\rho}^{\mathcal X}(x,u)=0\) if and only if
\((x,u)\) satisfies the KKT system of
\(\bigl\{\min_{x\in\mathcal X} f(x)\;\; \mathrm{s.t.}\;\; c(x)\le0\bigr\}\).
\end{proposition}

\begin{proof}
\noindent
\textit{(i) Necessity.}
If \(\mathcal R_{\alpha,\boldsymbol\rho}^{\mathcal X}(x,u)=0\), then
\(d_{\boldsymbol\rho}(x,u)=0\). Proposition~\ref{prop:residual_zero_set} gives
\(c(x)\le0\), \(u\ge0\), and \(u_i c_i(x)=0\). Moreover, from
\eqref{eq:pressure_residual}, \(d_{\boldsymbol\rho}(x,u)=0\) implies
\(\lambda_{\boldsymbol\rho}(x,u)=u\). The condition
\(\mathcal G_{\alpha,\boldsymbol\rho}^{\mathcal X}(x,u)=0\) yields
\(x = \Pi_{\mathcal X}\bigl(x-\alpha[\nabla f(x)+J_c(x)^\top u]\bigr)\).
By the optimality condition of Euclidean projection onto a convex set
\citep{bertsekas1982constrained,bauschke2017convex},
\(z=\Pi_{\mathcal X}(y)\) is equivalent to \(y-z\in N_{\mathcal X}(z)\),  where
\(N_{\mathcal X}(z)\) denotes the normal cone to \(\mathcal X\) at \(z\).
Taking \(z=x\) and \(y=x-\alpha[\nabla f(x)+J_c(x)^\top u]\) gives
\(-\alpha[\nabla f(x)+J_c(x)^\top u]\in N_{\mathcal X}(x)\), i.e.,
\(0\in \nabla f(x)+J_c(x)^\top u + N_{\mathcal X}(x)\). Thus the full KKT system
holds.

\noindent
\textit{(ii) Sufficiency.}
If \((x,u)\) satisfies the KKT system, then
Proposition~\ref{prop:residual_zero_set} gives \(d_{\boldsymbol\rho}(x,u)=0\),
and from \eqref{eq:pressure_residual} we have
\(\lambda_{\boldsymbol\rho}(x,u)=u\). The stationarity condition is equivalent to
\(x = \Pi_{\mathcal X}\bigl(x-\alpha[\nabla f(x)+J_c(x)^\top\lambda_{\boldsymbol\rho}(x,u)]\bigr)\).
Hence \(\mathcal G_{\alpha,\boldsymbol\rho}^{\mathcal X}(x,u)=0\) and
\(\mathcal R_{\alpha,\boldsymbol\rho}^{\mathcal X}(x,u)=0\).
\end{proof}

These results identify the object controlled by the feedback loop:
\(d_{\boldsymbol\rho}\) controls feasibility and complementarity, while
\(\mathcal G_{\alpha,\boldsymbol\rho}^{\mathcal X}\) controls primal
stationarity. This is exactly the residual-level interpretation used by RCML:
the multiplier memory is not driven by an arbitrary error signal, but by a
projected-pressure residual whose zero set is aligned with the KKT conditions.

\subsection{Finite-Gain Convex-Affine Backbone}
\label{subsec:theory_finite_gain_backbone}

We now analyze the deterministic residual-integral backbone behind
\textsc{Residual-I} and the finite-gain RCML update in
Table~\ref{tab:rcml_variants}. For the remainder of
Sections~\ref{subsec:theory_finite_gain_backbone} and
\ref{subsec:theory_stochastic_bound}, we analyze the isotropic specialization
of the coordinate-wise pressure scale: \(\boldsymbol\rho=\rho_0\mathbf 1\)
with \(\rho_0>0\). The symbol \(\rho_0\) is used only in this isotropic
analysis. This avoids mixing the vector-valued pressure scale used in
Algorithm~\ref{alg:rcml_unified} with the scalar weight needed in the Lyapunov
function. Under this specialization,
\(\boldsymbol\rho\odot(Ax-b)=\rho_0(Ax-b)\), and the projected pressure becomes
\([u+\rho_0(Ax-b)]_+\). 

Consider
\[
\min_{x\in\mathbb R^d}\; f(x)
~~
\mathrm{s.t.}~~
Ax-b\le0,
\]
where \(f\) is continuously differentiable, \(\mu\)-strongly convex, and has an
\(L\)-Lipschitz gradient. Assume Slater's condition holds
\citep{nocedal2006numerical}
and \(u_0\ge0\). Let
\(x^\star\) be the unique primal optimizer and let \(\mathcal U^\star\) be the
set of KKT multipliers. For any vector \(v\) and closed set \(\mathcal S\), we
write \(\operatorname{dist}(v,\mathcal S)=\inf_{w\in\mathcal S}\|v-w\|\).

Applying the template \eqref{eq:pd_template} with the signals from
\eqref{eq:pressure_residual} yields the deterministic finite-gain backbone
\begin{equation}
\label{eq:finite_gain_backbone}
\lambda_k=[u_k+\rho_0(Ax_k-b)]_+,~
d_k=\lambda_k-u_k,~
x_{k+1}=x_k-\alpha_k[\nabla f(x_k)+A^\top\lambda_k],~
u_{k+1}=u_k+\beta_k d_k .
\end{equation}
Here \(\lambda_k\) and \(d_k\) are deterministic instances of
\eqref{eq:pressure_residual} with \(c(x)=Ax-b\) and
\(\boldsymbol\rho=\rho_0\mathbf 1\). The effective memory-tracking gain
\(\beta_k\) corresponds to the product \(\eta_k\kappa_{\text{I}}\) in
\eqref{eq:rcml_memory_update}. We fix a gain ratio \(\theta>0\) and set
\(\beta_k=\theta\alpha_k\). Then \(u_{k+1}=u_k+\theta\alpha_k d_k\), and the
finite-gain condition \(0<\beta_k\le1\) becomes \(\theta\alpha_k\le1\). Under
this condition, the multiplier update can be rewritten as
\(u_{k+1}=(1-\beta_k)u_k+\beta_k\lambda_k\in\mathbb R_+^m\).

\begin{lemma}[Weighted summability and slow variation]
\label{lem:slow_variation}
Let \(\{a_k\}\) be a nonnegative sequence, \(\alpha_k>0\),
\(\sum_k\alpha_k=\infty\), and \(\alpha_k\to0\). If
\(\sum_k\alpha_k a_k<\infty\) and \(|a_{k+1}-a_k|\le C\alpha_k\), then
\(a_k\to0\).
\end{lemma}

\begin{proof}
Suppose, for contradiction, that \(a_k\not\to0\). Then there exist
\(\varepsilon>0\) and a subsequence \(\{k_j\}\) with \(a_{k_j}\ge\varepsilon\)
for all \(j\).  From \(|a_{k+1}-a_k|\le C\alpha_k\), for any \(k>k_j\) we have
\(|a_k-a_{k_j}| \le C\sum_{t=k_j}^{k-1}\alpha_t\).
Let \(\delta = \varepsilon/(2C)\).  Because \(\sum_k\alpha_k=\infty\) and
\(\alpha_k\to0\), for each sufficiently large \(k_j\) we can find a forward
index \(\ell_j > k_j\) such that
\[
\sum_{t=k_j}^{\ell_j}\alpha_t \ge \delta,~~
\sum_{t=k_j}^{\ell_j-1}\alpha_t < \delta .
\]
For every \(k\in[k_j,\ell_j]\), the total step mass is at most \(\delta\), hence
\(|a_k-a_{k_j}| \le C\delta = \varepsilon/2\), and therefore
\(a_k \ge \varepsilon/2\).
The \(\alpha\)-mass of the window \([k_j,\ell_j]\) is at least \(\delta\).
Since \(\sum_k\alpha_k=\infty\) and \(\alpha_k\to0\), we can select a subsequence
of disjoint such windows whose total \(\alpha\)-mass is infinite.
Consequently,
\[
\sum_k \alpha_k a_k
\ge \sum_j \frac{\varepsilon}{2}\sum_{t=k_j}^{\ell_j}\alpha_t
\ge \sum_j \frac{\varepsilon}{2}\,\delta = \infty,
\]
contradicting the hypothesis \(\sum_k\alpha_k a_k<\infty\).  Hence \(a_k\to0\).
\end{proof}

\begin{theorem}[Finite-gain convergence of the residual-integral backbone]
\label{thm:finite_gain_convergence}
Under the assumptions \(0<\alpha_k\le\bar\alpha\), \(0<\beta_k\le1\),
\(\sum_k\alpha_k=\infty\), and \(\sum_k\alpha_k^2<\infty\), for every
\(u^\star\in\mathcal U^\star\) the sequence generated by
\eqref{eq:finite_gain_backbone} satisfies
\[
x_k\to x^\star,~~
d_k\to0,~~
\nabla f(x_k)+A^\top\lambda_k\to0,
\]
and consequently
\[
\operatorname{dist}(u_k,\mathcal U^\star)\to0,~~
\operatorname{dist}(\lambda_k,\mathcal U^\star)\to0 .
\]
If \(\mathcal U^\star=\{u^\star\}\) is a singleton, then \(u_k\to u^\star\) and
\(\lambda_k\to u^\star\).
\end{theorem}

\begin{proof}
\proofstep{Step 1: weighted Lyapunov function}
For any \(u^\star\in\mathcal U^\star\), define
\[
V_k^\theta(u^\star)
:=
\|x_k-x^\star\|^2
+
(\theta\rho_0)^{-1}\|u_k-u^\star\|^2 .
\]
The weight \((\theta\rho_0)^{-1}\) matches the gain ratio
\(\beta_k/\alpha_k=\theta\). Let \(G_k=\nabla f(x_k)+A^\top\lambda_k\).

\proofstep{Step 2: primal descent expansion}
The primal update in \eqref{eq:finite_gain_backbone} gives
\[
\|x_{k+1}-x^\star\|^2
=
\|x_k-x^\star-\alpha_kG_k\|^2
=
\|x_k-x^\star\|^2
-2\alpha_k\langle x_k-x^\star,G_k\rangle
+\alpha_k^2\|G_k\|^2 .
\]
Since \(u^\star\in\mathcal U^\star\), KKT stationarity gives
\(\nabla f(x^\star)+A^\top u^\star=0\). Hence
\[
\begin{aligned}
\langle x_k-x^\star,G_k\rangle
&=
\langle x_k-x^\star,\nabla f(x_k)-\nabla f(x^\star)\rangle
+
\langle A(x_k-x^\star),\lambda_k-u^\star\rangle .
\end{aligned}
\]
By \(\mu\)-strong convexity of \(f\),
\(\langle x_k-x^\star,\nabla f(x_k)-\nabla f(x^\star)\rangle
\ge \mu\|x_k-x^\star\|^2\).
Therefore
\[
\begin{aligned}
\|x_{k+1}-x^\star\|^2
\le\;
\|x_k-x^\star\|^2
-2\alpha_k\mu\|x_k-x^\star\|^2 
-2\alpha_k\langle A(x_k-x^\star),\lambda_k-u^\star\rangle
+\alpha_k^2\|G_k\|^2 .
\end{aligned}
\]

\proofstep{Step 3: multiplier expansion and projection inequality}
From \eqref{eq:finite_gain_backbone}, the multiplier update is
\(u_{k+1}=u_k+\beta_k d_k\). Hence
\[
\|u_{k+1}-u^\star\|^2
=
\|u_k-u^\star\|^2
+ 2\beta_k\langle u_k-u^\star,d_k\rangle
+ \beta_k^2\|d_k\|^2 .
\]
Using \(d_k=\lambda_k-u_k\), we have
\(u_k-u^\star = \lambda_k-u^\star - d_k\), so that
\[
\langle u_k-u^\star,d_k\rangle
=
\langle\lambda_k-u^\star,d_k\rangle - \|d_k\|^2 .
\]

To bound \(\langle\lambda_k-u^\star,d_k\rangle\), we use the firm
nonexpansiveness of Euclidean projection onto a closed convex set
\citep{bauschke2017convex}.
Applying this property to
\(\lambda_k = [u_k+\rho_0(Ax_k-b)]_+\) and the KKT identity
\(u^\star = [u^\star+\rho_0(Ax^\star-b)]_+\) (which follows from
Proposition~\ref{prop:residual_zero_set}) yields
\[
\|\lambda_k-u^\star\|^2
\le
\bigl\langle
\lambda_k-u^\star,\;
(u_k+\rho_0(Ax_k-b)) - (u^\star+\rho_0(Ax^\star-b))
\bigr\rangle .
\]
Expanding the right-hand side via
\(u_k-u^\star = \lambda_k-u^\star - d_k\) gives
\[
\|\lambda_k-u^\star\|^2
\le
\|\lambda_k-u^\star\|^2
- \langle\lambda_k-u^\star, d_k\rangle
+ \rho_0\langle\lambda_k-u^\star, A(x_k-x^\star)\rangle .
\]
Canceling \(\|\lambda_k-u^\star\|^2\) from both sides, we obtain
\[
\langle\lambda_k-u^\star,d_k\rangle
\le
\rho_0\langle\lambda_k-u^\star, A(x_k-x^\star)\rangle .
\]
Substituting this into the expression for \(\langle u_k-u^\star,d_k\rangle\) gives
\[
\langle u_k-u^\star,d_k\rangle
\le
\rho_0\langle\lambda_k-u^\star, A(x_k-x^\star)\rangle - \|d_k\|^2 .
\]

\proofstep{Step 4: cross-term cancellation}
Multiply the multiplier expansion by \((\theta\rho_0)^{-1}\) and use
\(\beta_k = \theta\alpha_k\):
\[
\begin{aligned}
(\theta\rho_0)^{-1}\|u_{k+1}-u^\star\|^2
\le\;
(\theta\rho_0)^{-1}\|u_k-u^\star\|^2
+ 2\alpha_k\langle\lambda_k-u^\star, A(x_k-x^\star)\rangle
- \frac{2\alpha_k}{\rho_0}\|d_k\|^2
+ \frac{\theta\alpha_k^2}{\rho_0}\|d_k\|^2 .
\end{aligned}
\]
Adding this to the primal descent inequality cancels the cross term
\(\langle A(x_k-x^\star),\lambda_k-u^\star\rangle\) and yields
\[
\begin{aligned}
V_{k+1}^\theta(u^\star)
\le\;
V_k^\theta(u^\star)
- 2\alpha_k\mu\|x_k-x^\star\|^2
- \frac{2\alpha_k}{\rho_0}\|d_k\|^2
+ \alpha_k^2\|G_k\|^2
+ \frac{\theta\alpha_k^2}{\rho_0}\|d_k\|^2 .
\end{aligned}
\]

\proofstep{Step 5: bounding the quadratic terms}
From the definition of \(\lambda_k\) and the firm nonexpansiveness of the
projection,
\[
\|\lambda_k-u^\star\|
\le
\|u_k-u^\star + \rho_0 A(x_k-x^\star)\|
\le
\|u_k-u^\star\| + \rho_0\|A\|\|x_k-x^\star\| .
\]
Using \(G_k = \nabla f(x_k) - \nabla f(x^\star) + A^\top(\lambda_k-u^\star)\) and
the Lipschitz continuity of \(\nabla f\),
\[
\|G_k\|
\le L\|x_k-x^\star\| + \|A\|\|\lambda_k-u^\star\|
\le (L+\rho_0\|A\|^2)\|x_k-x^\star\| + \|A\|\|u_k-u^\star\| .
\]
Thus
\[
\|G_k\|^2
\le
2(L+\rho_0\|A\|^2)^2\|x_k-x^\star\|^2 + 2\|A\|^2\|u_k-u^\star\|^2 .
\]
From the definition of \(V_k^\theta(u^\star)\),
\(\|x_k-x^\star\|^2 \le V_k^\theta(u^\star)\) and
\(\|u_k-u^\star\|^2 \le \theta\rho_0 V_k^\theta(u^\star)\). Hence
\[
\|G_k\|^2 \le C_G V_k^\theta(u^\star),
\]
with \(C_G = 2(L+\rho_0\|A\|^2)^2 + 2\theta\rho_0\|A\|^2\).

For the \(\|d_k\|^2\) terms, note that \(\theta\alpha_k\le 1\) implies
\(\theta\alpha_k^2 \le \alpha_k\). Therefore
\[
-\frac{2\alpha_k}{\rho_0}\|d_k\|^2
+ \frac{\theta\alpha_k^2}{\rho_0}\|d_k\|^2
=
-\frac{\alpha_k}{\rho_0}\|d_k\|^2
- \frac{\alpha_k(1-\theta\alpha_k)}{\rho_0}\|d_k\|^2
\le
-\frac{\alpha_k}{\rho_0}\|d_k\|^2 .
\]
Substituting these bounds into the Lyapunov recursion gives
\begin{equation}
\label{eq:finite_gain_descent}
V_{k+1}^\theta(u^\star)
\le
\bigl(1 + C_G\alpha_k^2\bigr)V_k^\theta(u^\star)
- 2\alpha_k\mu\|x_k-x^\star\|^2
- \rho_0^{-1}\alpha_k\|d_k\|^2 .
\end{equation}

\proofstep{Step 6: convergence of \(x_k\), \(d_k\), and \(G_k\)}
Since \(\sum_k\alpha_k^2<\infty\), the Robbins--Siegmund almost-supermartingale lemma \citep{robbins1971convergence} applied to 
\eqref{eq:finite_gain_descent} guarantees that \(V_k^\theta(u^\star)\) remains
bounded and
\[
\sum_{k=0}^{\infty}\alpha_k\|x_k-x^\star\|^2 < \infty,~~
\sum_{k=0}^{\infty}\alpha_k\|d_k\|^2 < \infty .
\]
Boundedness of \(V_k^\theta\) implies boundedness of
\(\{x_k\}\), \(\{u_k\}\), \(\{\lambda_k\}\), \(\{d_k\}\), and \(\{G_k\}\).
The increments satisfy
\[
\|x_{k+1}-x_k\| = \alpha_k\|G_k\| = O(\alpha_k),~~
\|u_{k+1}-u_k\| = \theta\alpha_k\|d_k\| = O(\alpha_k).
\]

To bound \(\|d_{k+1}-d_k\|\), observe from \eqref{eq:pressure_residual} that
\(d_{\rho_0}(x,u) = [u+\rho_0(Ax-b)]_+ - u\). The projection \([\cdot]_+\) is
firmly nonexpansive, hence 1-Lipschitz. For any \((x,u)\) and \((x',u')\),
\[
\|d(x,u)-d(x',u')\|
\le 2\|u-u'\| + \rho_0\|A\|\|x-x'\| .
\]
Thus \(\|d_{k+1}-d_k\| = O(\alpha_k)\) on the bounded trajectory.
Similarly,
\[
\bigl|\|x_{k+1}-x^\star\|^2 - \|x_k-x^\star\|^2\bigr|
\le
(\|x_{k+1}-x^\star\| + \|x_k-x^\star\|)\,\|x_{k+1}-x_k\|
= O(\alpha_k).
\]
Applying Lemma~\ref{lem:slow_variation} first to
\(a_k=\|x_k-x^\star\|^2\) and then to \(a_k=\|d_k\|^2\), together with the
slow-variation bounds above, yields
\[
\|x_k-x^\star\|^2\to0,
~~
\|d_k\|^2\to0.
\]
Therefore \(x_k\to x^\star\) and \(d_k\to0\).

It remains to prove \(G_k\to0\). Since \(\nabla f\) is \(L\)-Lipschitz and the
trajectory is bounded, there exists \(C_G'>0\) such that
\[
\|G_{k+1}-G_k\|
\le C_G'\bigl(\|x_{k+1}-x_k\| + \|u_{k+1}-u_k\|\bigr)
\le C_G'\alpha_k .
\]
Suppose \(G_k\not\to0\). Then there exist \(\varepsilon>0\) and infinitely many
indices \(k_j\) with \(\|G_{k_j}\|\ge 4\varepsilon\). Set
\(v_j = G_{k_j}/\|G_{k_j}\|\) and \(\Delta = \varepsilon/(2C_G')\).
Since \(\sum_k\alpha_k = \infty\) and \(\alpha_k\to0\), for each sufficiently
large \(k_j\) we can find \(\ell_j\ge k_j\) such that
\(\sum_{k=k_j}^{\ell_j}\alpha_k \ge \Delta\) and
\(\sum_{k=k_j}^{\ell_j-1}\alpha_k < \Delta\).
For \(k\in[k_j,\ell_j]\), the slow variation of \(G_k\) yields
\(\|G_k-G_{k_j}\| \le C_G'\sum_{t=k_j}^{k-1}\alpha_t \le \varepsilon\), hence
\(\langle G_k, v_j\rangle \ge 3\varepsilon\).
Using \(x_{k+1}=x_k-\alpha_k G_k\), we get
\(\langle x_{\ell_j+1}-x_{k_j}, v_j\rangle
= -\sum_{k=k_j}^{\ell_j}\alpha_k\langle G_k,v_j\rangle
\le -3\varepsilon\Delta\).
Thus \(\|x_{\ell_j+1}-x_{k_j}\|\ge 3\varepsilon\Delta\), contradicting
\(x_k\to x^\star\). Therefore \(G_k\to0\).

Finally, let \(u_{k_j}\to \bar u\) be any cluster point of \(\{u_k\}\).
Since \(d_k = \lambda_k-u_k\to0\), we also have \(\lambda_{k_j}\to \bar u\).
Taking limits in \(G_{k_j} = \nabla f(x_{k_j}) + A^\top\lambda_{k_j}\to0\) gives
\(\nabla f(x^\star) + A^\top\bar u = 0\).
From \(d_k\to0\) and \(x_k\to x^\star\) we obtain
\(\bar u = [\bar u + \rho_0(Ax^\star-b)]_+\).
By Proposition~\ref{prop:residual_zero_set}, \(\bar u\ge0\),
\(Ax^\star-b\le0\), and \(\bar u_i(Ax^\star-b)_i = 0\) for all \(i\).
Thus every cluster point belongs to \(\mathcal U^\star\).
Since \(\{u_k\}\) is bounded and all its cluster points lie in the closed set
\(\mathcal U^\star\), we conclude
\(\operatorname{dist}(u_k,\mathcal U^\star)\to0\).
Because \(\lambda_k - u_k = d_k\to0\), the same holds for \(\lambda_k\).
If \(\mathcal U^\star = \{u^\star\}\), pointwise convergence follows.
\end{proof}

The theorem shows that full projected multiplier replacement is not required for
convergence. Finite-gain residual tracking preserves the primal--dual
cancellation once the multiplier energy is weighted by \((\theta\rho_0)^{-1}\).
This explains why the gain ratio \(\theta=\beta_k/\alpha_k\) is not merely a
tuning parameter; it determines the correct relative scale between primal descent
and multiplier memory tracking.

\begin{remark}[Time-varying gain ratios]
\label{rem:time_varying_gain}
The proof above uses a fixed gain ratio \(\theta=\beta_k/\alpha_k\). The same
argument can be extended to a slowly varying ratio
\(\theta_k=\beta_k/\alpha_k\), provided that
\[
0<\underline\theta\le \theta_k\le \overline\theta<\infty,~~
0<\beta_k\le1,~~
\sum_{k=0}^{\infty} |\theta_{k+1}^{-1}-\theta_k^{-1}|<\infty .
\]
In that case, one uses the time-varying Lyapunov function
\(V_k^{\theta_k} = \|x_k-x^\star\|^2 + (\theta_k\rho_0)^{-1}\|u_k-u^\star\|^2\).
After the same primal--dual cancellation, the only additional term is
\(\rho_0^{-1}(\theta_{k+1}^{-1}-\theta_k^{-1})\|u_{k+1}-u^\star\|^2\).
Taking the positive part and using the uniform bounds on \(\theta_k\), this term
is controlled by a summable multiplicative perturbation of the Lyapunov
recursion. Hence the Robbins--Siegmund almost-supermartingale argument
\citep{robbins1971convergence} remains valid. The summability condition on
\(|\theta_{k+1}^{-1}-\theta_k^{-1}|\) is essential; without it, the variation
of the Lyapunov weight may accumulate and destroy the quasi-Fej\'er descent
structure \citep{Combettes2001QuasiFejer}.
\end{remark}

\subsection{Stochastic Convex-Affine Residual Bound}
\label{subsec:theory_stochastic_bound}

We now quantify how stochastic mini-batch feedback perturbs the finite-gain
backbone. Throughout this subsection, the pressure scale is the isotropic scalar
\(\rho_0>0\) and the constraint is \(c(x)=Ax-b\). The analysis follows the
standard stochastic-approximation convention of representing mini-batch
quantities as conditionally unbiased estimates plus martingale-difference noise
\citep{robbins1951stochastic,kushner2003stochastic,borkar2008stochastic,bottou2018optimization}.
Specifically, assume $\hat g_k = \nabla f(x_k) + \zeta_k,~
\hat c_k = Ax_k - b + \varepsilon_k,$ where \(\mathcal F_k\) is the sigma-field generated by all randomness before
sampling the mini-batch at iteration \(k\). The noise variables satisfy $\mathbb E[\zeta_k\mid\mathcal F_k]=0,~
\mathbb E[\|\zeta_k\|^2\mid\mathcal F_k]\le \sigma_g^2/B_k,~\mathbb E[\varepsilon_k\mid\mathcal F_k]=0,~
\mathbb E[\|\varepsilon_k\|^2\mid\mathcal F_k]\le \sigma_c^2/B_k .$ Here \(\zeta_k\in\mathbb R^d\) denotes the mini-batch gradient noise,
\(\varepsilon_k\in\mathbb R^m\) denotes the mini-batch constraint-estimation noise, \(B_k\) is the batch size at iteration \(k\), and
\(\sigma_g,\sigma_c\) are uniform conditional second-moment constants. The \(1/B_k\) scaling corresponds to mini-batch averaging under bounded single-sample second moments; the theorem only uses the displayed conditional moment bounds, not independence across different iterations. The stochastic pressure and residual are $\hat\lambda_k =[u_k+\rho_0\hat c_k]_+,~\hat d_k = \hat\lambda_k - u_k,$
and the stochastic update is $x_{k+1} = x_k - \alpha_k[\hat g_k + A^\top\hat\lambda_k],~u_{k+1} = u_k + \theta\alpha_k\hat d_k .$

Let $\lambda_k = [u_k+\rho_0(Ax_k-b)]_+,~d_k = \lambda_k - u_k,~p_k = \hat\lambda_k - \lambda_k .$ By firm nonexpansiveness of the projection, $\|p_k\| \le \rho_0\|\varepsilon_k\|,~~
\mathbb E[\|p_k\|^2\mid\mathcal F_k] \le \rho_0^2\sigma_c^2/B_k .$

Throughout the proof, \(C,C_0,C_1,\ldots\) denote finite positive constants whose values may change from line to line. Constants denoted by \(K_1,\ldots,K_4\) in the theorem statement are horizon-independent constants after summation. Their allowed dependence is stated explicitly in the theorem. To avoid imposing global boundedness as an unverified assumption, we state the result for a stopped process. For \(R>0\), define $\tau_R = \inf\{k\ge 0 : \|x_k-x^\star\| + \|u_k-u^\star\| > R\}.$ The stopped iterates are \(x_k^R := x_{k\wedge\tau_R}\) and \(u_k^R := u_{k\wedge\tau_R}\). Equivalently, the stopped process follows the original stochastic update on \(\{k<\tau_R\}\) and has zero increment on \(\{k\ge\tau_R\}\). In the following, we omit the superscript \(R\) for readability, but all conditional recursions are understood in this stopped sense.
If the implementation uses compact primal projection and bounded dual clipping, then \(\tau_R=\infty\) for a sufficiently large \(R\). The following stopped finite-horizon bound is a residual version of the standard almost-supermartingale and stochastic-approximation argument used for noisy gradient-type recursions \citep{robbins1971convergence,kushner2003stochastic,borkar2008stochastic}.

\begin{theorem}[Stopped stochastic weighted residual bound]
\label{thm:stochastic_residual_bound}
Suppose the problem assumptions of Theorem~\ref{thm:finite_gain_convergence}
hold for the convex-affine backbone, and let \(u^\star\in\mathcal U^\star\) be
fixed. For a finite horizon \(T\), assume \(0<\theta\alpha_k\le1\) for
\(k=0,\ldots,T-1\). Then there exist constants
\(K_1,K_2,K_3,K_4>0\), depending on \(\mu\), \(L\), \(\|A\|\), \(\rho_0\),
\(\theta\), and \(R\), but not on \(T\), \(B_k\), or the noise realization, such
that
\begin{equation}
\label{eq:stochastic_weighted_bound}
\begin{aligned}
&\frac{
\sum_{k=0}^{T-1}
\alpha_k\,
\mathbb E\!\bigl[\ind_{\{k<\tau_R\}}
(\|x_k-x^\star\|^2+\|d_k\|^2)\bigr]
}{
\sum_{k=0}^{T-1}\alpha_k
}
\le
\frac{
K_1V_0^\theta + K_2\sum_{k=0}^{T-1}\alpha_k^2
}{
\sum_{k=0}^{T-1}\alpha_k
}\\
&+
K_3\,
\frac{
\sum_{k=0}^{T-1}\alpha_k\,\rho_0\sigma_c/\sqrt{B_k}
}{
\sum_{k=0}^{T-1}\alpha_k
}+
K_4\,
\frac{
\sum_{k=0}^{T-1}\alpha_k^2\,(\sigma_g^2+\rho_0^2\sigma_c^2)/B_k
}{
\sum_{k=0}^{T-1}\alpha_k
}.
\end{aligned}
\end{equation}
In particular, if \(\alpha_k = T^{-1/2}\) and \(B_k\equiv B\), then
\[
T^{-1}\sum_{k=0}^{T-1}
\mathbb E\!\bigl[\ind_{\{k<\tau_R\}}
(\|x_k-x^\star\|^2+\|d_k\|^2)\bigr]
\le
O(T^{-1/2}) + O(\rho_0\sigma_c/\sqrt{B})
+ O\bigl((\sigma_g^2+\rho_0^2\sigma_c^2)/(B\sqrt{T})\bigr).
\]
\end{theorem}

\begin{proof}
\proofstep{Step 1: stochastic primal expansion}
Let \(\hat G_k = \hat g_k + A^\top\hat\lambda_k\).
Since \(\hat g_k = \nabla f(x_k)+\zeta_k\) and
\(\hat\lambda_k = \lambda_k + p_k\), we have
\(\hat G_k = G_k + \zeta_k + A^\top p_k\) with
\(G_k = \nabla f(x_k) + A^\top\lambda_k\). The primal update gives
\[
\begin{aligned}
\|x_{k+1}-x^\star\|^2
&= \|x_k-x^\star - \alpha_k\hat G_k\|^2  \\
&= \|x_k-x^\star\|^2
- 2\alpha_k\langle x_k-x^\star,G_k\rangle
- 2\alpha_k\langle x_k-x^\star,\zeta_k+A^\top p_k\rangle
+ \alpha_k^2\|\hat G_k\|^2 .
\end{aligned}
\]
The \(G_k\)-term yields the same deterministic descent and cross term as in
Theorem~\ref{thm:finite_gain_convergence}. The gradient noise satisfies
\(\mathbb E[\langle x_k-x^\star,\zeta_k\rangle\mid\mathcal F_k]=0\).
On the stopped region, \(\|x_k-x^\star\|\le R\), so
\(2\alpha_k|\langle x_k-x^\star,A^\top p_k\rangle|
\le 2\alpha_k R\|A\|\|p_k\|\).

\proofstep{Step 2: stochastic multiplier expansion}
The stochastic multiplier update is
\(u_{k+1} = u_k + \theta\alpha_k(d_k + p_k)\).
Expanding the weighted multiplier energy gives
\[
\begin{aligned}
(\theta\rho_0)^{-1}\|u_{k+1}-u^\star\|^2
= (\theta\rho_0)^{-1}\|u_k-u^\star\|^2
+ \frac{2\alpha_k}{\rho_0}\langle u_k-u^\star,d_k\rangle
+ \frac{2\alpha_k}{\rho_0}\langle u_k-u^\star,p_k\rangle
+ \frac{\theta\alpha_k^2}{\rho_0}\|d_k+p_k\|^2 .
\end{aligned}
\]
The \(d_k\)-term is handled by the same projection inequality as in
Theorem~\ref{thm:finite_gain_convergence}, which cancels the primal cross term.
The pressure-noise term is bounded on the stopped region by
\(2\alpha_k\rho_0^{-1}|\langle u_k-u^\star,p_k\rangle|
\le 2\alpha_k R\rho_0^{-1}\|p_k\|\).

\proofstep{Step 3: conditional Lyapunov recursion}
Adding the primal and multiplier expansions, the deterministic cross terms
cancel. The remaining quadratic terms are bounded as follows. Using
\(\hat G_k = G_k + \zeta_k + A^\top p_k\),
\[
\|\hat G_k\|^2 \le 3\|G_k\|^2 + 3\|\zeta_k\|^2 + 3\|A\|^2\|p_k\|^2,
\]
and \(\|d_k+p_k\|^2 \le 2\|d_k\|^2 + 2\|p_k\|^2\). On the stopped region,
\(V_k^\theta\) is bounded by a constant depending on \(R\), and hence
\(\|G_k\|\) and \(\|d_k\|\) are bounded. Taking conditional expectations and using Young's inequality together with
the martingale-difference moment bounds on \(\zeta_k\) and \(p_k\) yields
\[
\mathbb E\bigl[\|\hat G_k\|^2 + \|d_k+p_k\|^2 \mid \mathcal F_k\bigr]
\le C(1 + V_k^\theta) + C(\sigma_g^2+\rho_0^2\sigma_c^2)/B_k
\]
for some constant \(C>0\). Therefore
\[
\begin{aligned}
\mathbb E[V_{k+1}^\theta\mid\mathcal F_k]
\le\;&
(1+C_0\alpha_k^2)V_k^\theta
- C_1\alpha_k(\|x_k-x^\star\|^2 + \|d_k\|^2)  \\
&+ C_2\alpha_k\,\mathbb E[\|p_k\|\mid\mathcal F_k]
+ C_3\alpha_k^2(\sigma_g^2+\rho_0^2\sigma_c^2)/B_k
+ C_4\alpha_k^2 .
\end{aligned}
\]
By firm nonexpansiveness of Euclidean projection
\citep{bauschke2017convex} and Jensen's inequality,
\(\mathbb E[\|p_k\|\mid\mathcal F_k]
\le (\mathbb E[\|p_k\|^2\mid\mathcal F_k])^{1/2}
\le \rho_0\sigma_c/\sqrt{B_k}\).

\proofstep{Step 4: stopped summation}
By the stopped-process construction, the update is applied only on
\(\{k<\tau_R\}\) and the increment is zero on \(\{k\ge\tau_R\}\). Hence the
conditional recursion becomes
\[
\begin{aligned}
\mathbb E[V_{k+1}^\theta \mid \mathcal F_k]
\le\;&
V_k^\theta
+ \ind_{\{k<\tau_R\}}
\Bigl[
C_0\alpha_k^2 V_k^\theta
- C_1\alpha_k(\|x_k-x^\star\|^2+\|d_k\|^2) \\
&~~
+ C_2\alpha_k\rho_0\sigma_c/\sqrt{B_k}
+ C_3\alpha_k^2(\sigma_g^2+\rho_0^2\sigma_c^2)/B_k
+ C_4\alpha_k^2
\Bigr].
\end{aligned}
\]
Taking unconditional expectations,
\[
\begin{aligned}
\mathbb E[V_{k+1}^\theta]
\le\;&
\mathbb E[V_k^\theta]
- C_1\alpha_k\,\mathbb E\!\bigl[\ind_{\{k<\tau_R\}}
(\|x_k-x^\star\|^2+\|d_k\|^2)\bigr] \\
&+ C_0\alpha_k^2\,\mathbb E[\ind_{\{k<\tau_R\}}V_k^\theta]
+ C_2\alpha_k\rho_0\sigma_c/\sqrt{B_k}+ C_3\alpha_k^2(\sigma_g^2+\rho_0^2\sigma_c^2)/B_k
+ C_4\alpha_k^2 .
\end{aligned}
\]
On the stopped region, \(V_k^\theta\) is bounded, so the term
\(C_0\alpha_k^2\,\mathbb E[\ind_{\{k<\tau_R\}}V_k^\theta]\) is absorbed into
\(C\alpha_k^2\). Summing from \(k=0\) to \(T-1\) gives
\[
\begin{aligned}
\sum_{k=0}^{T-1}
\alpha_k\,
\mathbb E\!\bigl[\ind_{\{k<\tau_R\}}
(\|x_k-x^\star\|^2+\|d_k\|^2)\bigr]
\le\;&
K_1 V_0^\theta + K_2\sum_{k=0}^{T-1}\alpha_k^2 + K_3\sum_{k=0}^{T-1}\alpha_k\rho_0\sigma_c/\sqrt{B_k} \\
&+ K_4\sum_{k=0}^{T-1}\alpha_k^2(\sigma_g^2+\rho_0^2\sigma_c^2)/B_k .
\end{aligned}
\]
Dividing by \(\sum_{k=0}^{T-1}\alpha_k\) yields
\eqref{eq:stochastic_weighted_bound}. This is the usual finite-horizon
weighted-average consequence of a one-step almost-supermartingale recursion
\citep{robbins1971convergence,kushner2003stochastic}. The stated rate follows by
substituting \(\alpha_k = T^{-1/2}\) and \(B_k\equiv B\).
\end{proof}

The bound contains a non-vanishing pressure-noise floor
\(O(\rho_0\sigma_c/\sqrt{B})\) when the batch size \(B\) is fixed. Thus, under
only second-moment mini-batch noise assumptions, the stochastic iterates are
controlled in a residual neighborhood rather than guaranteed to converge exactly
to the KKT set. Exact residual convergence can be recovered by increasing
\(B_k\) so that \(\rho_0\sigma_c/\sqrt{B_k}\to0\), or by using a vanishing
effective pressure scale when such a schedule is compatible with the desired
constraint accuracy. In the present algorithm, \(\rho_0\) is kept bounded away
from zero to maintain feasibility pressure, so the practical mechanisms for
reducing this neighborhood are variance reduction, larger batches, and filtering.

Even if \(\hat c_k\) is an unbiased estimate of \(c(x_k)\), the projection
\([\cdot]_+\) is nonlinear:
\[
\mathbb E[\hat c_k\mid\mathcal F_k] = c(x_k)
\;\;\not\Rightarrow\;\;
\mathbb E\bigl[[u_k+\rho_0\hat c_k]_+\mid\mathcal F_k\bigr]
= [u_k+\rho_0 c(x_k)]_+ .
\]
Therefore the main stochastic theorem uses the conservative first-moment pressure
bound. A sharper \(O(1/B)\) pressure-noise floor is available only when the
trajectory is locally away from the nonsmooth projection boundary.

\begin{proposition}[Local second-moment refinement away from projection kinks]
\label{prop:local_second_moment_refinement}
Suppose that, on the stopped region, there exists a margin \(\delta>0\) such that
\(|u_{k,i}+\rho_0 c_i(x_k)| \ge \delta\) for every coordinate \(i\). Assume also
that \(\|\rho_0\varepsilon_k\|_\infty \le \delta/2\) almost surely. Then the
active pattern of the projection is unchanged by the stochastic perturbation,
and there exists an \(\mathcal F_k\)-measurable diagonal matrix \(D_k\) with
diagonal entries in \(\{0,1\}\) such that \(p_k = D_k\rho_0\varepsilon_k\).
Consequently,
\[
\mathbb E[p_k\mid\mathcal F_k] = 0,~~
\mathbb E[\|p_k\|^2\mid\mathcal F_k] \le \rho_0^2\sigma_c^2/B_k .
\]
Under this local active-pattern stability condition, the first-moment pressure
term in \eqref{eq:stochastic_weighted_bound} disappears, and the pressure-noise
contribution is of order \(O(\rho_0^2\sigma_c^2/B)\) for fixed \(B\).
\end{proposition}

\begin{proof}
For each coordinate \(i\), define \(a_{k,i} = u_{k,i} + \rho_0 c_i(x_k)\).
The margin assumption gives \(|a_{k,i}|\ge\delta\), and the bounded-noise
assumption gives \(|\rho_0\varepsilon_{k,i}|\le\delta/2\). Therefore
\(a_{k,i}\) and \(a_{k,i}+\rho_0\varepsilon_{k,i}\) have the same sign.
If \(a_{k,i}>0\), then
\(p_{k,i} = [a_{k,i}+\rho_0\varepsilon_{k,i}]_+ - [a_{k,i}]_+
= \rho_0\varepsilon_{k,i}\).
If \(a_{k,i}<0\), then \(p_{k,i}=0\).
Thus \(p_k = D_k\rho_0\varepsilon_k\), where \(D_k\) selects the locally active
pressure coordinates. Since \(D_k\) is \(\mathcal F_k\)-measurable and
\(\mathbb E[\varepsilon_k\mid\mathcal F_k]=0\), we obtain
\(\mathbb E[p_k\mid\mathcal F_k]=0\) and
\(\mathbb E[\|p_k\|^2\mid\mathcal F_k]
= \rho_0^2\,\mathbb E[\|D_k\varepsilon_k\|^2\mid\mathcal F_k]
\le \rho_0^2\sigma_c^2/B_k\).
Substituting this into the proof of
Theorem~\ref{thm:stochastic_residual_bound} removes the first-moment
pressure-bias term and leaves only the second-moment pressure term.
\end{proof}

\subsection{Filtering as a Noise--Lag Tradeoff}
\label{subsec:theory_filtering}

The stochastic bound above shows that constraint noise enters mainly through the
pressure channel. Filtering acts before pressure formation and trades variance
reduction for tracking lag. We continue to work under the convex-affine setting
\(c(x)=Ax-b\) with isotropic pressure scale \(\rho_0\).
Using the filtered signal defined in \eqref{eq:constraint_filter}, write $\hat c_k = c(x_k) + \varepsilon_k,~
e_k = \bar c_k - c(x_k).$
Where \(e_k\) is the filtering error, \(\bar c_k\) is the exponentially smoothed constraint signal, and \(\gamma\) is the smoothing coefficient. We also use \(\alpha_{\max}:=\sup_k\alpha_k\) and a stopped-region bound \(M_G\) on the second moment of the stochastic primal direction. Assume \(c\) is \(L_c\)-Lipschitz with \(L_c = \|A\|\) (since \(c(x)=Ax-b\)), $\mathbb E[\varepsilon_k\mid\mathcal F_k] = 0,~\mathbb E[\|\varepsilon_k\|^2\mid\mathcal F_k] \le \sigma_c^2/B .$
For the stopped process inSection~\ref{subsec:theory_stochastic_bound}, the stochastic primal direction satisfies \(\mathbb E[\|\hat G_k\|^2\mid\mathcal F_k] \le M_G^2\). This holds because \(\hat G_k = G_k + \zeta_k + A^\top p_k\), and on the stopped region every component has bounded second moment by the definitions of \(\zeta_k\), \(p_k\), and the boundedness of \(G_k\) (which follows from the boundedness of \(V_k^\theta\) as shown in Theorem~\ref{thm:finite_gain_convergence}). The following bound is a direct EMA bias--variance decomposition under a martingale-difference noise model, as commonly used in stochastic approximation and stochastic-gradient analyses \citep{kushner2003stochastic,borkar2008stochastic,bottou2018optimization}.

\begin{proposition}[Filtering error bound in expectation]
\label{prop:filtering_error_bound}
For constant \(\gamma\in(0,1]\), the filtering error satisfies
\begin{equation}
\label{eq:filter_error_bound}
\mathbb E\|e_k\|^2
\le
C(1-\gamma)^{2k}\|e_0\|^2
+ C L_c^2 M_G^2 \alpha_{\max}^2 / \gamma^2
+ C\gamma\sigma_c^2/B .
\end{equation}
Moreover,
\[
\|[u_k+\rho_0\bar c_k]_+ - [u_k+\rho_0 c(x_k)]_+\|
\le \rho_0\|e_k\|,
\]
so the filtering error enters the stochastic residual bound through the pressure
channel.
\end{proposition}

\begin{proof}
From \eqref{eq:constraint_filter},
\(\bar c_k = (1-\gamma)\bar c_{k-1} + \gamma(c(x_k)+\varepsilon_k)\).
Subtracting \(c(x_k)\) and adding and subtracting \(c(x_{k-1})\) gives
\[
e_k = (1-\gamma)e_{k-1} + (1-\gamma)(c(x_{k-1})-c(x_k)) + \gamma\varepsilon_k .
\]
Unrolling the recursion,
\[
e_k = (1-\gamma)^k e_0
+ \sum_{t=1}^k (1-\gamma)^{k-t+1}(c(x_{t-1})-c(x_t))
+ \gamma\sum_{t=1}^k (1-\gamma)^{k-t}\varepsilon_t .
\]
The initial term contributes \((1-\gamma)^{2k}\|e_0\|^2\).
For the lag term, Minkowski's inequality gives
\[
\bigl(\mathbb E\|\text{lag}\|^2\bigr)^{1/2}
\le \sum_{t=1}^k (1-\gamma)^{k-t+1}
\bigl(\mathbb E\|c(x_{t-1})-c(x_t)\|^2\bigr)^{1/2}.
\]
By Lipschitz continuity and the bound \(\|x_t-x_{t-1}\|
\le \alpha_{t-1}\|\hat G_{t-1}\|\) (from the primal update with
nonexpansive projection), we have
\(\bigl(\mathbb E\|x_t-x_{t-1}\|^2\bigr)^{1/2}
\le \alpha_{t-1}M_G \le \alpha_{\max}M_G\).
Hence
\[
\bigl(\mathbb E\|\text{lag}\|^2\bigr)^{1/2}
\le L_c M_G \alpha_{\max} \sum_{j=0}^{k-1}(1-\gamma)^{j+1}
\le L_c M_G \alpha_{\max}/\gamma,
\]
and the squared lag contribution is bounded by
\(L_c^2 M_G^2 \alpha_{\max}^2/\gamma^2\).

For the stochastic noise term, the martingale-difference property gives
\[
\mathbb E\Bigl\|\gamma\sum_{t=1}^k (1-\gamma)^{k-t}\varepsilon_t\Bigr\|^2
= \gamma^2\sum_{t=1}^k (1-\gamma)^{2(k-t)}\mathbb E\|\varepsilon_t\|^2
\le \frac{\gamma}{2-\gamma}\,\frac{\sigma_c^2}{B}.
\]
Combining the three bounds and absorbing numerical constants yields
\eqref{eq:filter_error_bound}. The last inequality follows directly from the
firm nonexpansiveness of the projection \([\cdot]_+\).
\end{proof}

The bound has a direct interpretation. A smaller \(\gamma\) reduces the variance
term \(O(\gamma\sigma_c^2/B)\), but increases the lag term
\(O(\alpha_{\max}^2/\gamma^2)\). Thus filtering is a noise--lag tradeoff rather
than a uniformly monotone feasibility improvement. This matches the algorithmic
role described in Section~\ref{subsec:stochastic_components}: filtering
stabilizes the measurement channel before pressure formation, but excessive
smoothing can delay the response to changing constraints.

\subsection{Adaptive Scaling and Dynamic Correction as Feedback-Loop Perturbations}
\label{subsec:theory_scaling_dynamic}

Adaptive scaling and dynamic correction modify different channels of the
residual-controlled feedback loop. Adaptive scaling changes the map from the
constraint signal to the projected pressure, while residual-\(\nu\)PI correction
changes the memory-feedback signal after the residual has been formed. This
subsection does not claim a separate convergence theorem for these optional
modules. Instead, it gives a perturbation accounting result showing how such
modules enter the Lyapunov recursion of the convex-affine backbone.

\paragraph{Adaptive scaling as a pressure and residual perturbation}
In the implementation, the adaptive scale is computed from the clipped
second-moment rule in \eqref{eq:adaptive_scale}. For the present perturbation
analysis, it is enough to regard
\(\boldsymbol\rho_k = (\rho_{k,1},\ldots,\rho_{k,m})^\top\in\mathbb R_{++}^m\)
as a clipped coordinatewise scale. Let \(\rho_0\mathbf 1\) be the reference
isotropic scale used in the convex-affine backbone. For a filtered constraint
signal \(\bar c_k\), define
\[
\lambda_k^{\rho_0} = [u_k+\rho_0\bar c_k]_+,~~
d_k^{\rho_0} = \lambda_k^{\rho_0}-u_k,~~
\lambda_k^{\boldsymbol\rho_k} = [u_k+\boldsymbol\rho_k\odot\bar c_k]_+,~~
d_k^{\boldsymbol\rho_k} = \lambda_k^{\boldsymbol\rho_k}-u_k .
\]
The scale-induced perturbation is
\(p_k^\rho := \lambda_k^{\boldsymbol\rho_k} - \lambda_k^{\rho_0}\).
Since both residuals are formed by subtracting the same stored multiplier
\(u_k\), the same perturbation also appears in the memory residual:
\(d_k^{\boldsymbol\rho_k} = d_k^{\rho_0} + p_k^\rho\).
Thus adaptive scaling should be understood as a perturbation of both the primal
pressure channel and the memory-residual channel.

\begin{proposition}[Adaptive scaling as a projected-pressure perturbation]
\label{prop:adaptive_scaling_perturbation}
The scale-induced perturbation satisfies
\[
\|p_k^\rho\|
\le \|(\boldsymbol\rho_k - \rho_0\mathbf 1)\odot\bar c_k\|.
\]
Consequently, if \(\boldsymbol\rho_k\) is clipped and \(\bar c_k\) is bounded,
then adaptive scaling contributes a bounded perturbation to both the projected
pressure and the residual. If
\[
\frac{\sum_{k=0}^{T-1}\alpha_k\,
\|(\boldsymbol\rho_k-\rho_0\mathbf 1)\odot\bar c_k\|^2}
{\sum_{k=0}^{T-1}\alpha_k}
\to 0 ~~\text{as } T\to\infty,
\]
then the adaptive-scaling perturbation becomes negligible in the averaged
residual bound.
\end{proposition}

\begin{proof}
By firm nonexpansiveness of the projection \([\cdot]_+\),
\[
\|p_k^\rho\|
= \bigl\|[u_k+\boldsymbol\rho_k\odot\bar c_k]_+
- [u_k+\rho_0\bar c_k]_+\bigr\|
\le \|(\boldsymbol\rho_k-\rho_0\mathbf 1)\odot\bar c_k\|.
\]
If \(\boldsymbol\rho_k\in[\rho_{\min},\rho_{\max}]^m\) and \(\bar c_k\) is
bounded, the right-hand side is bounded. The last statement follows directly
from the same inequality after multiplying by \(\alpha_k\), summing, and
normalizing by \(\sum_{k=0}^{T-1}\alpha_k\).
\end{proof}

\paragraph{Dynamic residual correction as a memory-feedback perturbation}
Residual-\(\nu\)PI correction acts after the residual has been computed. Using
\eqref{eq:residual_filter} and \eqref{eq:nupi_update}, the memory signal can be
written as
\[
s_k = \kappa_{\text{I}} d_k + \kappa_{\text{P}}(\xi_k-\xi_{k-1})
= \kappa_{\text{I}}(d_k + q_k),~~
q_k := (\kappa_{\text{P}}/\kappa_{\text{I}})(\xi_k-\xi_{k-1}).
\]
The correction \(q_k\) is a memory-feedback perturbation. It does not change the
target residual equation \(d_k=0\), but the instantaneous condition \(d_k=0\)
alone does not imply \(s_k=0\) unless the auxiliary state has also settled. The
preserved equilibrium is the augmented one:
\(d_k=0,\; \xi_{k-1}=0\), under which \(\xi_k=0\), \(q_k=0\), and \(s_k=0\).
Therefore, the dynamic correction should be interpreted as a transient feedback
modification rather than as a change of the target KKT residual.

\paragraph{Generic inexact residual loop}
To account for stochastic gradients, pressure perturbations, residual
perturbations, and dynamic correction in a unified way, let
\(\lambda_k\) and \(d_k\) denote the reference isotropic backbone signals with
scale \(\rho_0\). Consider the inexact loop
\[
\tilde\lambda_k = \lambda_k + p_k^\lambda,~~
\tilde d_k = d_k + p_k^d + q_k,
\]
\[
x_{k+1} = x_k - \alpha_k\bigl[\nabla f(x_k) + A^\top\tilde\lambda_k + r_k\bigr],~~
u_{k+1} = u_k + \theta\alpha_k\tilde d_k .
\]
Here \(p_k^\lambda\) is the perturbation entering the primal pressure channel,
\(p_k^d\) is the perturbation entering the residual used by the memory update,
\(q_k\) is the additional memory-feedback correction induced by the
residual-\(\nu\)PI term, and \(r_k\) is the primal-gradient perturbation. For
perturbations generated by stochastic projection noise or adaptive scaling, one
typically has \(p_k^d = p_k^\lambda\), because the residual is defined as the
projected pressure minus \(u_k\). We keep the two symbols separate to allow more
general inexact implementations. This separation follows the standard inexact
fixed-point/perturbed stochastic approximation viewpoint
\citep{kushner2003stochastic,borkar2008stochastic,bauschke2017convex}.

The next statement is a quasi-Fej\'er perturbation accounting result: summable
first-order errors and summable second-order errors preserve the residual limit
of the reference recursion \citep{robbins1971convergence,bauschke2017convex}.

\begin{proposition}[Perturbation accounting for the inexact residual loop]
\label{prop:perturbation_accounting}
On a stopped Lyapunov region, the inexact residual loop satisfies
\[
\begin{aligned}
V_{k+1}^\theta
\le\;&
(1+C\alpha_k^2)V_k^\theta
- C_1\alpha_k(\|x_k-x^\star\|^2+\|d_k\|^2)+ C_2\alpha_k\bigl(\|p_k^\lambda\|+\|p_k^d\|+\|q_k\|+\|r_k\|\bigr) \\
&+ C_3\alpha_k^2\bigl(\|p_k^\lambda\|^2+\|p_k^d\|^2+\|q_k\|^2+\|r_k\|^2\bigr)
+ C_4\alpha_k^2 ,
\end{aligned}
\]
where the constants depend on the stopped-region radius and the problem
parameters, but not on \(k\). Consequently, if
\[
\sum_k \alpha_k\bigl(\|p_k^\lambda\|+\|p_k^d\|+\|q_k\|+\|r_k\|\bigr) < \infty
\]
and the corresponding second-order perturbation terms are summable, then the
residuals still satisfy \(x_k\to x^\star\), \(d_k\to0\), and \(G_k\to0\) as in
Theorem~\ref{thm:finite_gain_convergence}.
\end{proposition}

\begin{proof}
Repeat the proof of Theorem~\ref{thm:finite_gain_convergence} with
\(\lambda_k\) replaced by \(\tilde\lambda_k = \lambda_k + p_k^\lambda\),
\(d_k\) replaced by \(\tilde d_k = d_k + p_k^d + q_k\), and \(G_k\) replaced by
\(G_k + A^\top p_k^\lambda + r_k\), where
\(G_k = \nabla f(x_k) + A^\top\lambda_k\). The deterministic part gives the same
primal descent, multiplier expansion, projection inequality, and cross-term
cancellation as in the reference backbone.

The additional primal perturbation terms are
\[
-2\alpha_k\langle x_k-x^\star, A^\top p_k^\lambda + r_k\rangle
+ \alpha_k^2\|G_k + A^\top p_k^\lambda + r_k\|^2
- \alpha_k^2\|G_k\|^2 .
\]
On the stopped region, \(\|x_k-x^\star\|\le R\), so
\(2\alpha_k|\langle x_k-x^\star, A^\top p_k^\lambda + r_k\rangle|
\le 2\alpha_k R(\|A\|\|p_k^\lambda\| + \|r_k\|)\).
The difference of the quadratic terms is bounded by
\(C\alpha_k^2(V_k^\theta + \|p_k^\lambda\|^2 + \|r_k\|^2)\), where the
\(V_k^\theta\) term is absorbed into the usual \((1+C\alpha_k^2)V_k^\theta\)
factor.

For the multiplier update, set \(e_k := p_k^d + q_k\). Then
\(u_{k+1} = u_k + \theta\alpha_k(d_k + e_k)\).
Compared with the deterministic multiplier expansion, the additional terms in
the weighted multiplier energy are
\[
\frac{2\alpha_k}{\rho_0}\langle u_k-u^\star, e_k\rangle
+ \frac{\theta\alpha_k^2}{\rho_0}
\bigl(\|d_k+e_k\|^2 - \|d_k\|^2\bigr).
\]
The first term cannot be absorbed into the deterministic descent because the
Lyapunov recursion contains no negative term proportional to
\(\|u_k-u^\star\|^2\). Using the stopped-region bound
\(\|u_k-u^\star\|\le R\),
\(2\alpha_k\rho_0^{-1}|\langle u_k-u^\star, e_k\rangle|
\le 2\alpha_k R\rho_0^{-1}(\|p_k^d\|+\|q_k\|)\).
For the second term,
\(\|d_k+e_k\|^2 - \|d_k\|^2 = 2\langle d_k,e_k\rangle + \|e_k\|^2\).
Since \(d_k\) is bounded on the stopped region,
\(\alpha_k^2|\|d_k+e_k\|^2 - \|d_k\|^2|
\le C\alpha_k^2(\|p_k^d\|+\|q_k\|+\|p_k^d\|^2+\|q_k\|^2)\).
Because \(\alpha_k\le\bar\alpha\), the linear
\(\alpha_k^2(\|p_k^d\|+\|q_k\|)\) term is dominated by the first-order
perturbation term \(C\alpha_k(\|p_k^d\|+\|q_k\|)\).

Combining all bounds with the deterministic descent recursion yields the
stated inequality. If the first-order perturbations are \(\alpha_k\)-summable
and the second-order perturbations are summable, the same quasi-Fej\'er argument \citep{bauschke2017convex}
used in Theorem~\ref{thm:finite_gain_convergence} implies that the additional
inexact channels do not alter the asymptotic convergence of the residuals.
\end{proof}

The proposition is an accounting result, not a standalone convergence theorem
for adaptive scaling or residual-\(\nu\)PI correction. It identifies the type of
perturbation control that would be sufficient to preserve the deterministic
finite-gain residual limit. If the perturbations are merely bounded but not
summable, the same inequality should instead be interpreted as predicting a
residual neighborhood whose size depends on the accumulated perturbation level.

\paragraph{Remark on local nonconvex geometry and convergence}
The analysis above focuses on the convex-affine backbone, where the
projected-pressure residual dynamics admit a clean Lyapunov cancellation. For general nonlinear or nonconvex constrained problems, the same residual map
\[
F_{\boldsymbol\rho}(x,u)
=
[
\nabla f(x) + J_c(x)^\top\lambda_{\boldsymbol\rho}(x,u),~~
d_{\boldsymbol\rho}(x,u)
]^\top
\]
still has a meaningful local interpretation. Near a regular KKT point satisfying LICQ, strict complementarity, and the
second-order sufficient condition
\citep{bonnans2000perturbation,facchinei2003finite},
the active set is locally stable and the projected pressure map becomes
piecewise smooth with a fixed local pattern. In such a neighborhood,
\(d_{\boldsymbol\rho}(x,u)\) is locally equivalent to the
feasibility--complementarity residual, and \(F_{\boldsymbol\rho}(x,u)\) is
locally Lipschitz equivalent to the standard KKT residual after active-set
identification. Thus, under a local error bound or metric subregularity condition \citep{rockafellar1998variational,dontchev2014implicit},
small \(\|F_{\boldsymbol\rho}(x,u)\|\) implies a small local KKT residual.

A full nonconvex convergence theory would require additional ingredients beyond the scope of this paper. One possible route is to establish metric subregularity
of the KKT mapping around a regular solution and then prove that the
residual-controlled dynamics enter and remain in the corresponding neighborhood.
Another route is to combine the residual feedback law with a merit function,
trust-region mechanism, or line-search globalization scheme so that descent and
feasibility control can be enforced before local regularity becomes valid. For
stochastic nonconvex problems, one would additionally need to control mini-batch
noise, active-set switching, and the bias introduced by the projected pressure
near the kink set. These requirements are substantially stronger than the
convex-affine assumptions used in
Theorem~\ref{thm:finite_gain_convergence} and
Theorem~\ref{thm:stochastic_residual_bound}.

Therefore, we do not claim global convergence for general nonconvex constrained
learning problems. Instead, the role of the nonconvex experiments is empirical:
they test whether the residual-controlled multiplier dynamics provide stable
feasibility control and useful multiplier-memory behavior in settings where the
convex-affine theory no longer applies directly. The reported residual,
feasibility, and objective trajectories should be interpreted as evidence of
practical stability rather than as a substitute for a full nonconvex convergence
theorem.

\subsection{Summary of the Theoretical Findings}
\label{subsec:theory_summary}

The analysis establishes four conclusions. First,
\(d_{\boldsymbol\rho}(x,u)\) is an exact feasibility--complementarity residual
induced by the inequality augmented Lagrangian \eqref{eq:ineq_aug_lag}, and
together with projected stationarity it characterizes the KKT system. Second, the
convex-affine backbone \eqref{eq:finite_gain_backbone} converges under
finite-gain tracking \(\beta_k=\theta\alpha_k\), with
\[
x_k\to x^\star,~~
d_k\to0,~~
\nabla f(x_k)+A^\top\lambda_k\to0,~~
\operatorname{dist}(u_k,\mathcal U^\star)\to0 .
\]
Third, the stopped stochastic analysis gives the weighted residual bound
\eqref{eq:stochastic_weighted_bound}. For \(\alpha_k=T^{-1/2}\) and fixed batch
size \(B\), the conservative finite-time behavior is
\[
O(T^{-1/2}) + O(\rho_0\sigma_c/\sqrt{B}),
\]
which indicates a residual neighborhood under fixed mini-batch noise. A sharper
\(O(\rho_0^2\sigma_c^2/B)\) pressure-noise floor is obtained only under a local
active-pattern stability condition that keeps the projected pressure away from
its kink set. Fourth, filtering, adaptive pressure scaling, and
residual-\(\nu\)PI correction act as perturbations of the measurement, pressure,
and memory-feedback channels, respectively. In particular, filtering reduces
pressure variance at the cost of lag, adaptive scaling changes the pressure gain,
and dynamic residual correction modifies the transient memory feedback.

\section{Experiments}
\label{sec:experiments}

\subsection{Experimental roadmap}
\label{subsec:exp_questions_protocol}

The experiments are designed to evaluate the residual-controlled multiplier
framework from four complementary perspectives. First, we test whether the
pressure-memory residual provides the intended signal-level behavior and whether
finite-gain memory tracking stabilizes multiplier dynamics. Second, we examine
whether the framework improves tolerance-aware decision quality under stochastic
reserve constraints. Third, we evaluate reliable solution discovery in a
nonconvex stochastic pricing-inventory problem with multiple random
initializations. Fourth, we test whether the same multiplier-control mechanism
can be integrated into neural fair-ranking training under exposure constraints.

Each experiment isolates a
specific role of the framework: projected pressure formation, finite-gain
residual memory, measurement filtering, adaptive pressure scaling, and dynamic
memory-feedback correction. Detailed data generation procedures, hardware,
hyperparameters, validation grids, sample sizes, and evaluation windows are
reported in~\ref{app:experimental_details}.

Unless otherwise stated, all reported values are averaged over random seeds, and
tail metrics are computed over the final evaluation window of each run. All
multiplier-based methods are implemented under the stochastic primal--dual
interface in \eqref{eq:pd_template}. They differ in how they construct the
effective pressure, the pressure-memory residual, and the memory feedback signal.

\subsection{Compared Methods and Metrics}
\label{subsec:exp_methods_metrics}

We compare RCML variants with unconstrained training and fixed penalty baselines
\citep{nocedal2006numerical}, raw violation-based
stochastic primal--dual updates including \textsc{SGDA-Signed} and
\textsc{SGDA-Positive}
\citep{xu2020primaldual,jin2022stochasticpd},
projected augmented-Lagrangian replacement \textsc{Projected-ALM}
\citep{bertsekas1982constrained},
constraint extrapolation \textsc{ConEx} \citep{boob2023stochastic},
perturbed primal--dual ascent \textsc{GDPA}
\citep{lu2022gdpa,alacaoglu2024complexity}, stochastic SQP \textsc{StoSQP}
\citep{curtis2024sequential,fang2024fully}, and stochastic barrier baselines
\textsc{StochBarrier} \citep{dimitrieski2025relaxedbarrier}.
The main residual-controlled methods are \textsc{Residual-I},
\textsc{RCML-Core}, \textsc{RCML-Adaptive}, and \textsc{RCML-Robust}.
Diagnostic variants such as \textsc{Adaptive-Proj.},
\textsc{RCML-Adaptive-}\(\tau\), and
\textsc{Residual-}\(\nu\)\textsc{PI} are used only in ablation studies.

Objective, cost, profit, and NDCG measure primal decision quality, where NDCG is
used for ranking evaluation \citep{jarvelin2002cumulated}. Constraint violation,
tolerance excess, feasibility rate, exposure violation, and reliable rate measure
constraint control, where exposure violation is used for exposure-based
fair-ranking constraints \citep{singh2018fairness}. Dual total variation measures
multiplier-memory movement, and residual total variation measures fluctuation of
the pressure-memory residual. Runtime measures the computational overhead of the
residual-controlled update.

\subsection{Experiment 1: Signal Diagnostics and Module Ablation}
\label{subsec:exp_core_diagnostics}

Experiment~1 verifies the residual-controlled multiplier interface through
signal-level diagnostics and module-level ablations. The signal-level diagnostic
uses a scalar four-phase constraint sequence to test whether different multiplier
signals distinguish violation activation, stale-memory release, and inactive
dead-zone behavior. The module-level ablation then evaluates the contribution of
projected pressure formation, finite-gain residual memory, constraint filtering,
and adaptive coordinatewise pressure scaling.

The ablation benchmark contains stochastic LP, convex QP, and mildly nonconvex QP
tasks. We evaluate stationary feedback, high-noise feedback, and
heterogeneous-scale constraints in the main text. The active-set switching
diagnostic is reported in ~\ref{app:additional_ablation_tables} because it
is used only to evaluate optional transient extensions.

\paragraph{Signal-level validation}
We first use a four-phase scalar constraint sequence with inactive, violated,
release, and inactive phases. The compared signals are
\[
s_k^{\rm signed}=\hat c_k,~~
s_k^{\rm positive}=[\hat c_k]_+,~~
s_k^{\rm residual}=[u_k+\rho_0\hat c_k]_+-u_k,
\]
and a filtered residual signal $\bar c_k=(1-\gamma)\bar c_{k-1}+\gamma\hat c_k,~s_k^{\rm filtered}=[u_k+\rho_0\bar c_k]_+-u_k.$
All signals are inserted into the same multiplier recursion $u_{k+1}=[u_k+\eta s_k]_+ .$
Figure~\ref{fig:basic_signal} shows that the signed signal can decrease obsolete
multiplier mass but also leaks negative control in inactive regions. The positive
signal avoids negative leakage but does not actively release obsolete multiplier
memory. The residual signal releases stored multiplier mass when \(u_k>0\) and
becomes zero after the multiplier reaches the inactive dead zone. The filtered
residual further suppresses near-boundary noise, with a visible response delay.

\begin{figure}[H]
    \centering
    \includegraphics[width=0.7\columnwidth]{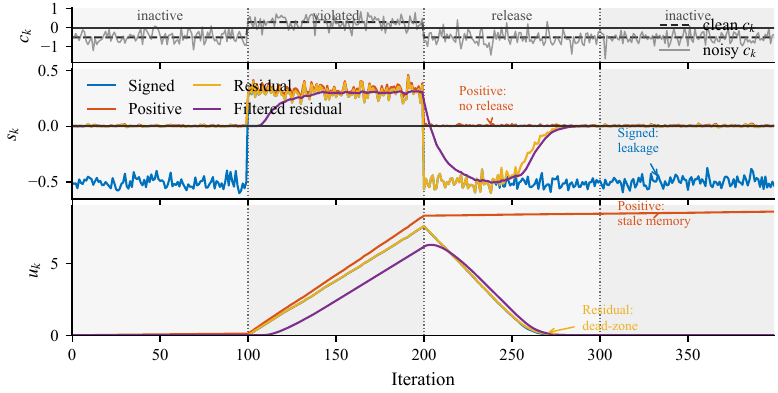}
    \caption{
    Signal-level validation of the residual multiplier interface.
    The top panel shows the scalar constraint state \(c_k\), the middle panel
    shows the control signal \(s_k\), and the bottom panel shows the multiplier
    state \(u_k\). The residual signal distinguishes feasible-but-releasing
    states from genuinely inactive dead-zone states.
    }
    \label{fig:basic_signal}
\end{figure}

\paragraph{Core ablation}
The core ablation compares raw violation-based updates, projected-pressure
replacement, finite-gain residual memory, and adaptive pressure scaling.

\begin{table}[H]
\centering
\caption{
Core signal and memory ablation on stationary stochastic LP, QP, and mildly nonconvex QP tasks.
RelRate denotes the fraction of tail iterates with violation below \(5\times10^{-2}\).
Lower is better for ObjTail, ViolTail, DualTV, ResidualTV, and runtime; higher is better for RelRate.
}
\label{tab:exp1_core_ablation}
\footnotesize
\setlength{\tabcolsep}{\exptabcolsep}
\renewcommand{\arraystretch}{\exptabstretch}
\begin{adjustbox}{max width=\textwidth,center}
\begin{tabular}{cccccccc}
\toprule
Problem & Method & ObjTail & ViolTail & RelRate & DualTV & ResidualTV & Runtime \\
\midrule
LP & \textsc{SGDA-Signed}
& \(-1.040{\times}10^{1}\)
& \(5.989{\times}10^{-1}\)
& \(0.000\)
& \(2.261{\times}10^{-3}\)
& \(2.188{\times}10^{0}\)
& \(6.39{\times}10^{-2}\) \\
LP & \textsc{SGDA-Positive}
& \(-7.166{\times}10^{0}\)
& \(2.201{\times}10^{0}\)
& \(0.000\)
& \(2.994{\times}10^{-3}\)
& \(9.535{\times}10^{-1}\)
& \(6.04{\times}10^{-2}\) \\
LP & \textsc{Projected-ALM}
& \(-1.045{\times}10^{1}\)
& \(3.715{\times}10^{-2}\)
& \(0.644\)
& \(6.909{\times}10^{-1}\)
& \(1.027{\times}10^{0}\)
& \(6.00{\times}10^{-2}\) \\
LP & \textsc{Residual-I}
& \(-1.052{\times}10^{1}\)
& \(4.339{\times}10^{-2}\)
& \(0.610\)
& \(3.533{\times}10^{-3}\)
& \(9.712{\times}10^{-1}\)
& \(6.21{\times}10^{-2}\) \\
LP & \textsc{RCML-Adaptive}
& \(-1.035{\times}10^{1}\)
& \(3.489{\times}10^{-2}\)
& \(0.704\)
& \(6.469{\times}10^{-3}\)
& \(1.265{\times}10^{0}\)
& \(6.89{\times}10^{-2}\) \\
\midrule
QP & \textsc{SGDA-Signed}
& \(7.161{\times}10^{-1}\)
& \(4.844{\times}10^{-1}\)
& \(0.000\)
& \(1.299{\times}10^{-3}\)
& \(6.999{\times}10^{-1}\)
& \(7.24{\times}10^{-2}\) \\
QP & \textsc{SGDA-Positive}
& \(1.578{\times}10^{1}\)
& \(1.315{\times}10^{-1}\)
& \(0.166\)
& \(6.881{\times}10^{-4}\)
& \(3.642{\times}10^{-1}\)
& \(7.23{\times}10^{-2}\) \\
QP & \textsc{Projected-ALM}
& \(7.893{\times}10^{-1}\)
& \(2.330{\times}10^{-2}\)
& \(0.832\)
& \(1.862{\times}10^{-1}\)
& \(2.717{\times}10^{-1}\)
& \(7.25{\times}10^{-2}\) \\
QP & \textsc{Residual-I}
& \(7.734{\times}10^{-1}\)
& \(1.547{\times}10^{-2}\)
& \(0.982\)
& \(1.717{\times}10^{-3}\)
& \(2.715{\times}10^{-1}\)
& \(7.39{\times}10^{-2}\) \\
QP & \textsc{RCML-Adaptive}
& \(7.824{\times}10^{-1}\)
& \(2.094{\times}10^{-2}\)
& \(0.886\)
& \(3.515{\times}10^{-3}\)
& \(4.134{\times}10^{-1}\)
& \(8.01{\times}10^{-2}\) \\
\midrule
NCVQP & \textsc{SGDA-Signed}
& \(-8.062{\times}10^{-2}\)
& \(2.447{\times}10^{-1}\)
& \(0.000\)
& \(4.765{\times}10^{-4}\)
& \(5.588{\times}10^{-1}\)
& \(7.15{\times}10^{-2}\) \\
NCVQP & \textsc{SGDA-Positive}
& \(3.105{\times}10^{0}\)
& \(9.024{\times}10^{-1}\)
& \(0.000\)
& \(1.375{\times}10^{-3}\)
& \(4.746{\times}10^{-1}\)
& \(7.26{\times}10^{-2}\) \\
NCVQP & \textsc{Projected-ALM}
& \(-8.840{\times}10^{-2}\)
& \(1.442{\times}10^{-3}\)
& \(1.000\)
& \(1.042{\times}10^{-1}\)
& \(1.561{\times}10^{-1}\)
& \(7.44{\times}10^{-2}\) \\
NCVQP & \textsc{Residual-I}
& \(-4.810{\times}10^{-2}\)
& \(2.269{\times}10^{-3}\)
& \(1.000\)
& \(5.898{\times}10^{-4}\)
& \(1.447{\times}10^{-1}\)
& \(7.34{\times}10^{-2}\) \\
NCVQP & \textsc{RCML-Adaptive}
& \(-6.371{\times}10^{-2}\)
& \(1.835{\times}10^{-3}\)
& \(1.000\)
& \(1.077{\times}10^{-3}\)
& \(1.938{\times}10^{-1}\)
& \(7.87{\times}10^{-2}\) \\
\bottomrule
\end{tabular}
\end{adjustbox}
\end{table}

Table~\ref{tab:exp1_core_ablation} gives three observations. First, raw
violation-based updates have poor reliability in all three tasks. Second,
projected replacement improves feasibility but produces much larger multiplier
variation. Third, replacing full replacement with finite-gain residual memory
keeps feasibility in the same range while substantially reducing DualTV. The
adaptive variant improves reliability in LP, remains competitive in QP, and
preserves full reliability in NCVQP.

\paragraph{High-noise feedback}
The high-noise setting evaluates how filtering changes the measurement channel.

\begin{table}[H]
\centering
\caption{
High-noise module ablation.
High-noise uses sparse constraint mini-batches with additive observation noise.
ViolP95 is the tail \(95\)-th percentile violation.
}
\label{tab:exp1_high_noise_ablation}
\footnotesize
\setlength{\tabcolsep}{\exptabcolsep}
\renewcommand{\arraystretch}{\exptabstretch}
\begin{adjustbox}{max width=\textwidth,center}
\begin{tabular}{cccccccc}
\toprule
Problem & Method & ViolTail & ViolP95 & RelRate & DualTV & ResidualTV & Runtime \\
\midrule
LP & \textsc{Projected-ALM}
& \(0.000\)
& \(0.000\)
& \(1.000\)
& \(1.909{\times}10^{0}\)
& \(3.022{\times}10^{0}\)
& \(1.39{\times}10^{-1}\) \\
LP & \textsc{Adaptive-Proj.}
& \(1.517{\times}10^{-2}\)
& \(1.236{\times}10^{-1}\)
& \(0.903\)
& \(1.077{\times}10^{0}\)
& \(1.138{\times}10^{0}\)
& \(1.58{\times}10^{-1}\) \\
LP & \textsc{Residual-I}
& \(0.000\)
& \(0.000\)
& \(1.000\)
& \(4.724{\times}10^{-3}\)
& \(2.422{\times}10^{0}\)
& \(1.43{\times}10^{-1}\) \\
LP & \textsc{RCML-Core}
& \(5.802{\times}10^{-2}\)
& \(1.967{\times}10^{-1}\)
& \(0.539\)
& \(3.039{\times}10^{-3}\)
& \(9.806{\times}10^{-1}\)
& \(1.46{\times}10^{-1}\) \\
LP & \textsc{RCML-Adaptive}
& \(1.863{\times}10^{-2}\)
& \(1.143{\times}10^{-1}\)
& \(0.855\)
& \(4.242{\times}10^{-3}\)
& \(1.438{\times}10^{0}\)
& \(1.60{\times}10^{-1}\) \\
\midrule
QP & \textsc{Projected-ALM}
& \(1.231{\times}10^{-4}\)
& \(0.000\)
& \(1.000\)
& \(5.098{\times}10^{-1}\)
& \(8.075{\times}10^{-1}\)
& \(1.94{\times}10^{-1}\) \\
QP & \textsc{Adaptive-Proj.}
& \(4.949{\times}10^{-3}\)
& \(3.179{\times}10^{-2}\)
& \(0.965\)
& \(4.004{\times}10^{-1}\)
& \(4.825{\times}10^{-1}\)
& \(2.20{\times}10^{-1}\) \\
QP & \textsc{Residual-I}
& \(1.044{\times}10^{-4}\)
& \(0.000\)
& \(1.000\)
& \(2.331{\times}10^{-3}\)
& \(6.703{\times}10^{-1}\)
& \(2.02{\times}10^{-1}\) \\
QP & \textsc{RCML-Core}
& \(2.982{\times}10^{-2}\)
& \(1.011{\times}10^{-1}\)
& \(0.740\)
& \(1.380{\times}10^{-3}\)
& \(2.394{\times}10^{-1}\)
& \(1.93{\times}10^{-1}\) \\
QP & \textsc{RCML-Adaptive}
& \(5.812{\times}10^{-3}\)
& \(4.555{\times}10^{-2}\)
& \(0.959\)
& \(2.893{\times}10^{-3}\)
& \(5.785{\times}10^{-1}\)
& \(2.15{\times}10^{-1}\) \\
\midrule
NCVQP & \textsc{Projected-ALM}
& \(0.000\)
& \(0.000\)
& \(1.000\)
& \(3.739{\times}10^{-1}\)
& \(6.179{\times}10^{-1}\)
& \(2.67{\times}10^{-1}\) \\
NCVQP & \textsc{Adaptive-Proj.}
& \(0.000\)
& \(0.000\)
& \(1.000\)
& \(3.017{\times}10^{-1}\)
& \(3.982{\times}10^{-1}\)
& \(2.86{\times}10^{-1}\) \\
NCVQP & \textsc{Residual-I}
& \(0.000\)
& \(0.000\)
& \(1.000\)
& \(9.441{\times}10^{-4}\)
& \(4.416{\times}10^{-1}\)
& \(2.51{\times}10^{-1}\) \\
NCVQP & \textsc{RCML-Core}
& \(1.229{\times}10^{-5}\)
& \(0.000\)
& \(1.000\)
& \(5.771{\times}10^{-4}\)
& \(1.725{\times}10^{-1}\)
& \(2.50{\times}10^{-1}\) \\
NCVQP & \textsc{RCML-Adaptive}
& \(0.000\)
& \(0.000\)
& \(1.000\)
& \(1.026{\times}10^{-3}\)
& \(3.495{\times}10^{-1}\)
& \(2.72{\times}10^{-1}\) \\
\bottomrule
\end{tabular}
\end{adjustbox}
\end{table}

Table~\ref{tab:exp1_high_noise_ablation} shows that filtering reduces residual
variation but may increase violation because of response lag. For instance, in
high-noise QP, \textsc{RCML-Core} has lower ResidualTV than \textsc{Residual-I},
while its violation is larger. In NCVQP, all projected-pressure and residual
variants satisfy the reliability criterion, and the main difference is variation
in multiplier and residual signals.

\paragraph{Heterogeneous-scale constraints}
The heterogeneous-scale setting evaluates adaptive coordinatewise pressure
scaling.

\begin{table}[H]
\centering
\caption{
Heterogeneous-scale module ablation.
Constraint channels are multiplied by unequal positive scales.
ViolP95 is the tail \(95\)-th percentile violation.
}
\label{tab:exp1_heterogeneous_ablation}
\footnotesize
\setlength{\tabcolsep}{\exptabcolsep}
\renewcommand{\arraystretch}{\exptabstretch}
\begin{adjustbox}{max width=\textwidth,center}
\begin{tabular}{cccccccc}
\toprule
Problem & Method & ViolTail & ViolP95 & RelRate & DualTV & ResidualTV & Runtime \\
\midrule
LP & \textsc{Adaptive-Proj.}
& \(2.524{\times}10^{-2}\)
& \(1.335{\times}10^{-1}\)
& \(0.884\)
& \(6.307{\times}10^{-1}\)
& \(7.454{\times}10^{-1}\)
& \(1.20{\times}10^{-1}\) \\
LP & \textsc{Residual-I}
& \(4.158{\times}10^{-2}\)
& \(6.722{\times}10^{-2}\)
& \(0.653\)
& \(3.817{\times}10^{-3}\)
& \(1.264{\times}10^{0}\)
& \(1.15{\times}10^{-1}\) \\
LP & \textsc{RCML-Core}
& \(1.233{\times}10^{-1}\)
& \(2.556{\times}10^{-1}\)
& \(0.006\)
& \(2.693{\times}10^{-3}\)
& \(5.692{\times}10^{-1}\)
& \(1.09{\times}10^{-1}\) \\
LP & \textsc{RCML-Adaptive}
& \(3.183{\times}10^{-2}\)
& \(1.436{\times}10^{-1}\)
& \(0.810\)
& \(4.393{\times}10^{-3}\)
& \(9.805{\times}10^{-1}\)
& \(1.20{\times}10^{-1}\) \\
\midrule
QP & \textsc{Adaptive-Proj.}
& \(1.239{\times}10^{-2}\)
& \(4.622{\times}10^{-2}\)
& \(0.953\)
& \(3.203{\times}10^{-1}\)
& \(3.390{\times}10^{-1}\)
& \(1.68{\times}10^{-1}\) \\
QP & \textsc{Residual-I}
& \(2.817{\times}10^{-2}\)
& \(7.310{\times}10^{-2}\)
& \(0.910\)
& \(1.554{\times}10^{-3}\)
& \(2.864{\times}10^{-1}\)
& \(1.58{\times}10^{-1}\) \\
QP & \textsc{RCML-Core}
& \(4.139{\times}10^{-2}\)
& \(8.928{\times}10^{-2}\)
& \(0.703\)
& \(9.301{\times}10^{-4}\)
& \(1.085{\times}10^{-1}\)
& \(1.59{\times}10^{-1}\) \\
QP & \textsc{RCML-Adaptive}
& \(1.613{\times}10^{-2}\)
& \(5.855{\times}10^{-2}\)
& \(0.937\)
& \(3.283{\times}10^{-3}\)
& \(4.106{\times}10^{-1}\)
& \(1.59{\times}10^{-1}\) \\
\midrule
NCVQP & \textsc{Adaptive-Proj.}
& \(1.039{\times}10^{-3}\)
& \(7.364{\times}10^{-3}\)
& \(1.000\)
& \(2.331{\times}10^{-1}\)
& \(2.454{\times}10^{-1}\)
& \(1.51{\times}10^{-1}\) \\
NCVQP & \textsc{Residual-I}
& \(4.244{\times}10^{-3}\)
& \(2.325{\times}10^{-2}\)
& \(1.000\)
& \(1.026{\times}10^{-3}\)
& \(1.722{\times}10^{-1}\)
& \(1.51{\times}10^{-1}\) \\
NCVQP & \textsc{RCML-Core}
& \(5.913{\times}10^{-3}\)
& \(2.918{\times}10^{-2}\)
& \(1.000\)
& \(8.341{\times}10^{-4}\)
& \(5.835{\times}10^{-2}\)
& \(1.55{\times}10^{-1}\) \\
NCVQP & \textsc{RCML-Adaptive}
& \(2.988{\times}10^{-3}\)
& \(1.887{\times}10^{-2}\)
& \(1.000\)
& \(1.689{\times}10^{-3}\)
& \(3.051{\times}10^{-1}\)
& \(1.59{\times}10^{-1}\) \\
\bottomrule
\end{tabular}
\end{adjustbox}
\end{table}

Table~\ref{tab:exp1_heterogeneous_ablation} shows that adaptive scaling improves
constraint handling when channels have unequal scales. On LP and QP,
\textsc{RCML-Adaptive} improves feasibility relative to \textsc{RCML-Core}. On
NCVQP, all residual and projected variants reach full reliability, and the
differences are mainly in violation magnitude and signal variation. These results
indicate that adaptive scaling is useful under heterogeneous constraints, but it
is not uniformly dominant across all task geometries.

\FloatBarrier

\subsection{Experiment 2: Tolerance-Aware Stochastic Energy-Reserve Allocation}
\label{subsec:exp_energy_reserve}

Experiment~2 evaluates RCML on a tolerance-aware stochastic energy-reserve
allocation task. The decision variable represents reserve allocation, and the
objective is to reduce operating cost while satisfying stochastic reserve
adequacy constraints. Unlike strict feasibility benchmarks, this task allows
small reserve inadequacy within a prescribed engineering tolerance. The relevant
goal is therefore to balance cost, tolerance feasibility, reserve
over-allocation, and multiplier stability.

For each method, we report tail-averaged operating cost, maximum reserve
inadequacy violation, tolerance excess, tolerance-feasible rate, tolerance-aware
score, reserve over-allocation, and runtime. The tolerance-aware score is defined
as
\[
\mathrm{TA}
=
\mathrm{Cost}
+
\omega_{\rm tol}[\mathrm{Viol}-\delta_{\rm tol}]_+ .
\]
The \textsc{RCML-Adaptive} configuration is selected by a small validation sweep
and then fixed for the main comparison. The sweep is reported in
Appendix~\ref{app:additional_ablation_tables}.

\begin{table}[H]
\centering
\caption{
Tolerance-aware stochastic energy-reserve allocation results.
Normal uses \(B=256\) and \(\delta_{\rm tol}=10^{-2}\); Stress uses \(B=96\) and \(\delta_{\rm tol}=2\times10^{-2}\).
}
\label{tab:exp2_energy_results}
\footnotesize
\setlength{\tabcolsep}{2.6pt}
\renewcommand{\arraystretch}{\exptabstretch}
\begin{adjustbox}{max width=\textwidth,center}
\begin{tabular}{ccccccccc}
\toprule
Setting
& Method
& Cost
& Viol.
& TolExc
& TolFeas
& TA
& OverRes.
& Time \\
\midrule
Normal
& \textsc{Primal-only}
& \(0.000\)
& \(1.805\)
& \(1.795\)
& \(0.000\)
& \(179.466\)
& \(0.000\)
& \(0.11\) \\
Normal
& \textsc{SGDA-Signed}
& \(5.454\)
& \(3.018{\times}10^{-2}\)
& \(2.018{\times}10^{-2}\)
& \(0.000\)
& \(7.472\)
& \(1.133{\times}10^{-1}\)
& \(0.10\) \\
Normal
& \textsc{SGDA-Positive}
& \(17.808\)
& \(0.000\)
& \(0.000\)
& \(1.000\)
& \(17.808\)
& \(2.176\)
& \(0.11\) \\
Normal
& \textsc{Projected-ALM}
& \(5.445\)
& \(1.684{\times}10^{-2}\)
& \(7.400{\times}10^{-3}\)
& \(0.187\)
& \(6.185\)
& \(1.073{\times}10^{-1}\)
& \(0.11\) \\
Normal
& \textsc{Residual-I}
& \(5.453\)
& \(8.182{\times}10^{-3}\)
& \(6.659{\times}10^{-4}\)
& \(0.727\)
& \(5.519\)
& \(1.111{\times}10^{-1}\)
& \(0.11\) \\
Normal
& \textsc{StoSQP}
& \(5.669\)
& \(8.523{\times}10^{-3}\)
& \(4.321{\times}10^{-3}\)
& \(0.680\)
& \(6.102\)
& \(1.640{\times}10^{-1}\)
& \(0.40\) \\
Normal
& \textsc{GDPA}
& \(\mathbf{5.294}\)
& \(4.680{\times}10^{-2}\)
& \(3.680{\times}10^{-2}\)
& \(0.000\)
& \(8.975\)
& \(8.270{\times}10^{-2}\)
& \(0.11\) \\
Normal
& \textsc{ConEx}
& \(5.457\)
& \(1.833{\times}10^{-2}\)
& \(8.350{\times}10^{-3}\)
& \(0.013\)
& \(6.292\)
& \(1.130{\times}10^{-1}\)
& \(0.11\) \\
Normal
& \textsc{StochBarrier}
& \(5.853\)
& \(1.126{\times}10^{-3}\)
& \(2.246{\times}10^{-4}\)
& \(0.967\)
& \(5.875\)
& \(2.097{\times}10^{-1}\)
& \(0.11\) \\
Normal
& \textsc{RCML-Adaptive}
& \(5.444\)
& \(7.335{\times}10^{-3}\)
& \(\mathbf{6.496{\times}10^{-4}}\)
& \(0.760\)
& \(\mathbf{5.509}\)
& \(1.064{\times}10^{-1}\)
& \(0.13\) \\
\midrule
Stress
& \textsc{Primal-only}
& \(0.000\)
& \(1.856\)
& \(1.836\)
& \(0.000\)
& \(183.608\)
& \(0.000\)
& \(0.08\) \\
Stress
& \textsc{SGDA-Signed}
& \(5.701\)
& \(3.254{\times}10^{-2}\)
& \(1.423{\times}10^{-2}\)
& \(0.227\)
& \(7.124\)
& \(1.167{\times}10^{-1}\)
& \(0.08\) \\
Stress
& \textsc{SGDA-Positive}
& \(17.808\)
& \(0.000\)
& \(0.000\)
& \(1.000\)
& \(17.808\)
& \(2.132\)
& \(0.08\) \\
Stress
& \textsc{Projected-ALM}
& \(5.713\)
& \(3.090{\times}10^{-2}\)
& \(1.323{\times}10^{-2}\)
& \(0.280\)
& \(7.035\)
& \(1.189{\times}10^{-1}\)
& \(0.08\) \\
Stress
& \textsc{Residual-I}
& \(5.692\)
& \(\mathbf{1.335{\times}10^{-2}}\)
& \(\mathbf{3.887{\times}10^{-4}}\)
& \(0.867\)
& \(\mathbf{5.731}\)
& \(1.137{\times}10^{-1}\)
& \(0.09\) \\
Stress
& \textsc{StoSQP}
& \(6.107\)
& \(1.839{\times}10^{-2}\)
& \(8.186{\times}10^{-3}\)
& \(0.607\)
& \(6.926\)
& \(2.175{\times}10^{-1}\)
& \(0.38\) \\
Stress
& \textsc{GDPA}
& \(\mathbf{5.539}\)
& \(4.883{\times}10^{-2}\)
& \(2.883{\times}10^{-2}\)
& \(0.000\)
& \(8.422\)
& \(8.731{\times}10^{-2}\)
& \(0.09\) \\
Stress
& \textsc{ConEx}
& \(5.701\)
& \(2.765{\times}10^{-2}\)
& \(1.013{\times}10^{-2}\)
& \(0.367\)
& \(6.714\)
& \(1.166{\times}10^{-1}\)
& \(0.09\) \\
Stress
& \textsc{StochBarrier}
& \(6.702\)
& \(0.000\)
& \(0.000\)
& \(1.000\)
& \(6.702\)
& \(3.477{\times}10^{-1}\)
& \(0.09\) \\
Stress
& \textsc{RCML-Adaptive}
& \(5.695\)
& \(1.447{\times}10^{-2}\)
& \(1.430{\times}10^{-3}\)
& \(0.727\)
& \(5.838\)
& \(1.138{\times}10^{-1}\)
& \(0.11\) \\
\bottomrule
\end{tabular}
\end{adjustbox}
\end{table}

Table~\ref{tab:exp2_energy_results} shows distinct tolerance-aware behaviors.
\textsc{Primal-only} has the lowest cost but violates reserve constraints.
\textsc{SGDA-Positive} and \textsc{StochBarrier} achieve high feasibility but
with larger over-reserve and higher cost. \textsc{GDPA} obtains the lowest cost
in both regimes, but its tolerance-feasible rate is zero. In the normal regime,
\textsc{RCML-Adaptive} gives the lowest tolerance-aware score among the
nontrivial methods. In the stress regime, \textsc{Residual-I} gives the lowest
tolerance-aware score, while \textsc{RCML-Adaptive} remains close with lower cost
than conservative feasible baselines.

\FloatBarrier

\subsection{Experiment 3: Reliable Solution Discovery in Nonconvex Pricing-Inventory Allocation}
\label{subsec:exp_nonconvex_pricing_inventory}

Experiment~3 evaluates reliable solution discovery in a nonconvex stochastic
pricing-inventory allocation problem. The decision variables are the inventory
allocation \(x_i\) and price \(r_i\) for each item. Demand depends nonlinearly on
price and stochastic demand shocks, and the constraint requires expected shortage
to remain below a prescribed tolerance level.

The purpose of this experiment is to test whether each method can reliably
discover admissible high-profit solutions under stochastic feedback and random
initialization. For each method and scenario, we run multiple seed--initialization
pairs. A run is called reliable if its tail tolerance-feasible rate is at least
\(0.9\). We report reliable rate, number of reliable runs, mean and best profit
over reliable runs, reliable violation, reliable overstock, and all-run average
profit.

\begin{table}[H]
\centering
\caption{
Reliable solution discovery in nonconvex stochastic pricing-inventory allocation.
ReliableRate is the fraction of seed--initialization runs whose tail tolerance-feasible rate is at least \(0.9\).
MeanRelProfit, BestRelProfit, RelViol., and RelOverstock are computed only over reliable runs.
AllProfit is averaged over all runs, including unreliable ones.
}
\label{tab:exp3_nonconvex_reliable_discovery}
\footnotesize
\setlength{\tabcolsep}{2.6pt}
\renewcommand{\arraystretch}{\exptabstretch}
\begin{adjustbox}{max width=\textwidth,center}
\begin{tabular}{ccccccccc}
\toprule
Setting
& Method
& ReliableRate
& Runs
& MeanRelProfit
& BestRelProfit
& RelViol.
& RelOverstock
& AllProfit \\
\midrule
Normal
& \textsc{Primal-only}
& \(0.000\)
& \(0/25\)
& --
& --
& --
& --
& \(50.0297\) \\
Normal
& \textsc{SGDA-Signed}
& \(0.920\)
& \(23/25\)
& \(50.0897\)
& \(52.2016\)
& \(7.865{\times}10^{-3}\)
& \(8.619{\times}10^{-3}\)
& \(50.0144\) \\
Normal
& \textsc{SGDA-Positive}
& \(0.920\)
& \(23/25\)
& \(50.0817\)
& \(52.1975\)
& \(7.480{\times}10^{-3}\)
& \(9.900{\times}10^{-3}\)
& \(50.0067\) \\
Normal
& \textsc{Projected-ALM}
& \(1.000\)
& \(25/25\)
& \(50.0060\)
& \(52.1763\)
& \(3.886{\times}10^{-3}\)
& \(8.155{\times}10^{-3}\)
& \(50.0060\) \\
Normal
& \textsc{Residual-I}
& \(1.000\)
& \(25/25\)
& \(50.0049\)
& \(52.1782\)
& \(2.011{\times}10^{-3}\)
& \(8.102{\times}10^{-3}\)
& \(50.0049\) \\
Normal
& \textsc{RCML-Adaptive}
& \(1.000\)
& \(25/25\)
& \(50.0058\)
& \(52.1790\)
& \(2.337{\times}10^{-3}\)
& \(\mathbf{7.976{\times}10^{-3}}\)
& \(50.0058\) \\
Normal
& \textsc{StoSQP}
& \(0.200\)
& \(5/25\)
& \(49.7684\)
& \(50.1769\)
& \(7.085{\times}10^{-3}\)
& \(1.017{\times}10^{-2}\)
& \(49.9650\) \\
Normal
& \textsc{GDPA}
& \(0.600\)
& \(15/25\)
& \(50.1948\)
& \(52.1967\)
& \(9.589{\times}10^{-3}\)
& \(8.144{\times}10^{-3}\)
& \(50.0210\) \\
Normal
& \textsc{ConEx}
& \(0.920\)
& \(23/25\)
& \(50.0908\)
& \(52.2007\)
& \(7.899{\times}10^{-3}\)
& \(8.444{\times}10^{-3}\)
& \(50.0155\) \\
Normal
& \textsc{StochBarrier}
& \(1.000\)
& \(25/25\)
& \(46.7575\)
& \(50.2614\)
& \(2.228{\times}10^{-3}\)
& \(2.341{\times}10^{-2}\)
& \(46.7575\) \\
\midrule
Stress
& \textsc{Primal-only}
& \(0.000\)
& \(0/25\)
& --
& --
& --
& --
& \(50.6661\) \\
Stress
& \textsc{SGDA-Signed}
& \(0.000\)
& \(0/25\)
& --
& --
& --
& --
& \(50.4946\) \\
Stress
& \textsc{SGDA-Positive}
& \(0.000\)
& \(0/25\)
& --
& --
& --
& --
& \(50.4849\) \\
Stress
& \textsc{Projected-ALM}
& \(1.000\)
& \(25/25\)
& \(48.3833\)
& \(50.7371\)
& \(8.883{\times}10^{-3}\)
& \(1.246{\times}10^{-3}\)
& \(48.3833\) \\
Stress
& \textsc{Residual-I}
& \(1.000\)
& \(25/25\)
& \(49.4066\)
& \(51.7168\)
& \(1.621{\times}10^{-2}\)
& \(5.550{\times}10^{-5}\)
& \(49.4066\) \\
Stress
& \textsc{RCML-Adaptive}
& \(1.000\)
& \(25/25\)
& \(\mathbf{49.6109}\)
& \(\mathbf{51.8196}\)
& \(1.833{\times}10^{-2}\)
& \(\mathbf{2.222{\times}10^{-5}}\)
& \(49.6109\) \\
Stress
& \textsc{StoSQP}
& \(0.560\)
& \(14/25\)
& \(47.5305\)
& \(49.4934\)
& \(1.355{\times}10^{-2}\)
& \(5.530{\times}10^{-3}\)
& \(47.3098\) \\
Stress
& \textsc{GDPA}
& \(0.000\)
& \(0/25\)
& --
& --
& --
& --
& \(50.5476\) \\
Stress
& \textsc{ConEx}
& \(0.000\)
& \(0/25\)
& --
& --
& --
& --
& \(50.5015\) \\
Stress
& \textsc{StochBarrier}
& \(1.000\)
& \(25/25\)
& \(44.9962\)
& \(47.5309\)
& \(3.066{\times}10^{-3}\)
& \(1.187{\times}10^{-2}\)
& \(44.9962\) \\
\bottomrule
\end{tabular}
\end{adjustbox}
\end{table}

Table~\ref{tab:exp3_nonconvex_reliable_discovery} shows that the normal regime is
less discriminative: several methods obtain high reliable rates. In the stress
regime, raw violation-based methods, \textsc{GDPA}, and \textsc{ConEx} have zero
reliable runs, whereas projected-pressure and residual-controlled methods remain
reliable. Among the fully reliable methods, \textsc{RCML-Adaptive} has the
highest mean reliable profit under stress, while \textsc{StochBarrier} is more
conservative. These observations are empirical and do not imply global nonconvex
convergence.

\FloatBarrier

\subsection{Experiment 4: Large-Scale Fair Ranking under Exposure Constraints}
\label{subsec:exp4_fair_ranking}

Experiment~4 evaluates whether residual-controlled multiplier updates can be used
as an in-processing constrained training method for neural fair ranking. Each
query contains candidate items with features and binary group attributes. A
neural ranker assigns a score to each item and is trained with a ListNet-type
ranking loss. The constraint controls the soft exposure gap between the two
groups.

Each query contains a set of candidate items with feature vectors and group
attributes. A neural ranker \(s_\theta(q,i)\) assigns a score to each item. The
ranking loss is denoted by $\mathcal L_{\rm rank}(\theta).$
Let \(E_0(\theta)\) and \(E_1(\theta)\) denote the average soft exposure assigned
to the two groups under a differentiable softmax ranking distribution. The soft
exposure gap is $g(\theta)=|E_0(\theta)-E_1(\theta)|.$
The fairness constraint is $c(\theta)=g(\theta)-\epsilon \le 0.$
In the reported experiment, evaluation uses an admissible tolerance band $g(\theta)\le \epsilon_{\rm eval}+\delta_{\rm tol}=0.015.$
For residual-controlled methods, a stricter internal training threshold is used
to absorb mini-batch noise, while all reported metrics are evaluated using the
common admissible threshold.

\begin{figure}[H]
    \centering
    \includegraphics[width=0.8\textwidth]{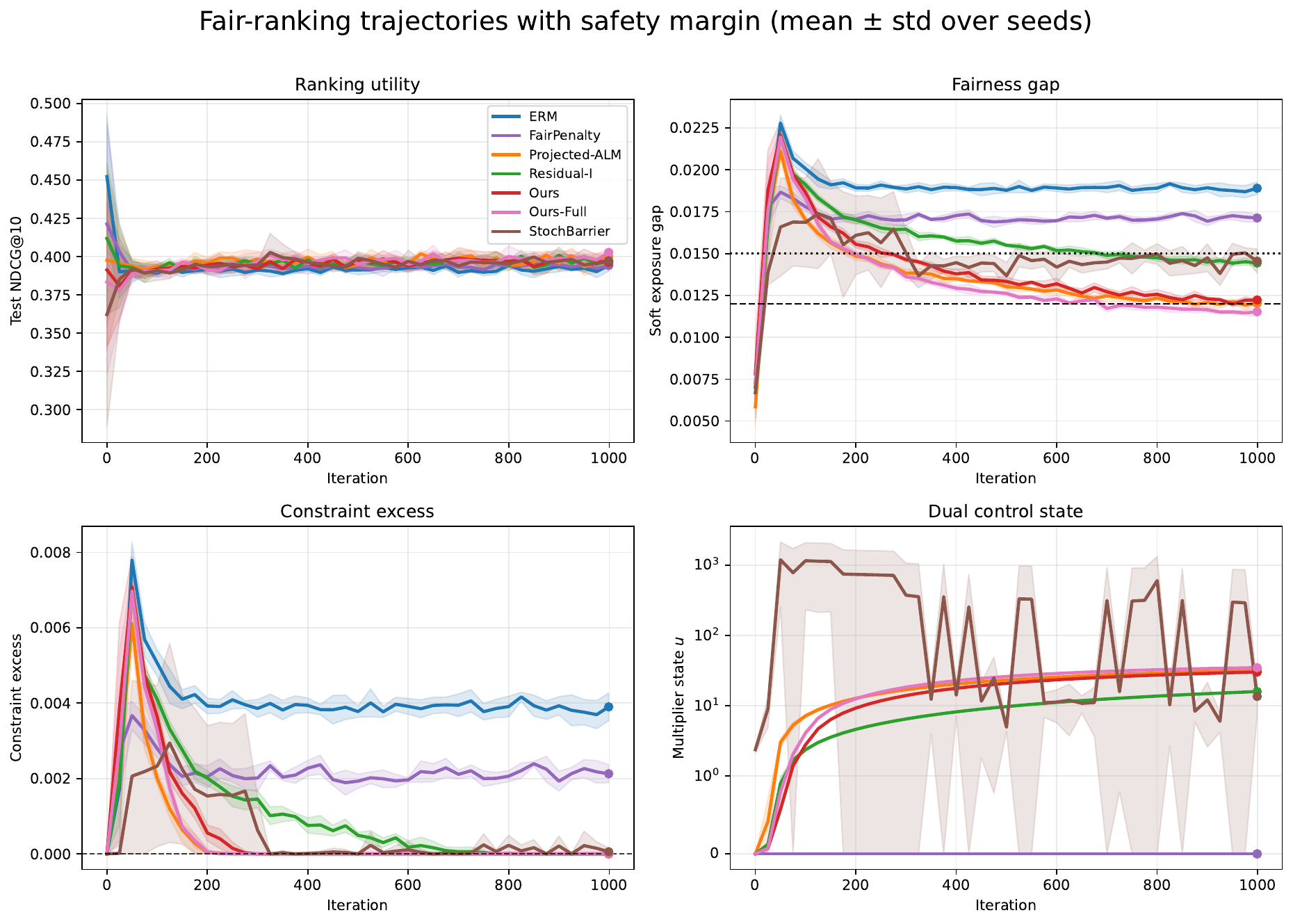}
    \caption{
    Large-scale fair-ranking training trajectories.
    Curves show mean \(\pm\) standard deviation over random seeds.
    The dashed line in the fairness-gap panel denotes the target exposure
    threshold, and the dotted line denotes the admissible evaluation threshold
    \(0.015\). SGDA-Signed and SGDA-Positive are omitted from the trajectory plot
    for readability but are included in Table~\ref{tab:exp4_fair_ranking}.
    }
    \label{fig:exp4_fair_ranking_traj}
\end{figure}

\begin{table}[H]
\centering
\caption{
Large-scale fair ranking under exposure constraints.
Evaluation uses admissible soft exposure gap \(0.015\).
All statistics are computed over random seeds at the final evaluation point.
}
\label{tab:exp4_fair_ranking}
\footnotesize
\setlength{\tabcolsep}{3.0pt}
\renewcommand{\arraystretch}{\exptabstretch}
\begin{adjustbox}{max width=\textwidth,center}
\begin{tabular}{cccccc}
\toprule
Method
& NDCG@10
& Soft gap
& TolExcess
& Dual state \(u\)
& Runtime (s) \\
\midrule
\textsc{ERM}
& \(0.3944 \pm 0.0039\)
& \(0.0189 \pm 0.0004\)
& \(3.901{\times}10^{-3} \pm 3.672{\times}10^{-4}\)
& \(0.0000 \pm 0.0000\)
& \(7.64 \pm 2.17\) \\
\textsc{FairPenalty}
& \(0.3942 \pm 0.0020\)
& \(0.0171 \pm 0.0002\)
& \(2.126{\times}10^{-3} \pm 2.398{\times}10^{-4}\)
& \(0.0000 \pm 0.0000\)
& \(8.67 \pm 1.81\) \\
\textsc{SGDA-Signed}
& \(0.3948 \pm 0.0037\)
& \(0.0183 \pm 0.0001\)
& \(3.304{\times}10^{-3} \pm 1.263{\times}10^{-4}\)
& \(1.3558 \pm 0.0100\)
& \(7.98 \pm 1.54\) \\
\textsc{SGDA-Positive}
& \(0.3944 \pm 0.0036\)
& \(0.0182 \pm 0.0001\)
& \(3.211{\times}10^{-3} \pm 1.308{\times}10^{-4}\)
& \(1.3435 \pm 0.0076\)
& \(7.96 \pm 1.61\) \\
\textsc{Projected-ALM}
& \(0.4002 \pm 0.0025\)
& \(0.0120 \pm 0.0001\)
& \(0.000{\times}10^{0} \pm 0.000{\times}10^{0}\)
& \(31.5923 \pm 0.1286\)
& \(8.29 \pm 1.76\) \\
\textsc{Residual-I}
& \(0.3971 \pm 0.0020\)
& \(0.0144 \pm 0.0001\)
& \(0.000{\times}10^{0} \pm 0.000{\times}10^{0}\)
& \(15.8298 \pm 0.1247\)
& \(8.66 \pm 2.27\) \\
\textsc{RCML-Adaptive}
& \(0.3962 \pm 0.0043\)
& \(0.0122 \pm 0.0001\)
& \(0.000{\times}10^{0} \pm 0.000{\times}10^{0}\)
& \(19.9738 \pm 0.1614\)
& \(8.25 \pm 2.72\) \\
\textsc{RCML-Robust}
& \(\mathbf{0.4028 \pm 0.0024}\)
& \(0.0115 \pm 0.0003\)
& \(0.000{\times}10^{0} \pm 0.000{\times}10^{0}\)
& \(14.6091 \pm 0.1905\)
& \(8.29 \pm 2.20\) \\
\textsc{StochBarrier}
& \(0.3973 \pm 0.0030\)
& \(0.0145 \pm 0.0007\)
& \(5.660{\times}10^{-5} \pm 9.761{\times}10^{-5}\)
& \(13.4958 \pm 6.2623\)
& \(8.94 \pm 2.29\) \\
\bottomrule
\end{tabular}
\end{adjustbox}
\end{table}

Figure~\ref{fig:exp4_fair_ranking_traj} and
Table~\ref{tab:exp4_fair_ranking} show that unconstrained and penalty-based
training retain positive tolerance excess. The SGDA variants also remain outside
the admissible exposure band. Projected-pressure and residual-controlled methods
satisfy the admissible exposure tolerance. Among feasible methods,
\textsc{RCML-Robust} gives the highest NDCG@10, while \textsc{Residual-I} uses a
smaller dual state than \textsc{Projected-ALM}.

\FloatBarrier

\subsection{Summary of Empirical Findings}
\label{subsec:exp_summary}

The experiments give four main observations. First, projected pressure formation
improves feasibility relative to raw violation-based multiplier updates. Second,
finite-gain residual memory substantially reduces multiplier variation compared
with full projected replacement. Third, filtering and adaptive scaling act on
different channels: filtering reduces measurement noise at the cost of lag,
whereas adaptive scaling is useful under heterogeneous constraint magnitudes.
Fourth, in allocation, nonconvex discovery, and fair-ranking tasks, the
residual-controlled variants avoid either infeasible low-cost solutions or overly
conservative feasible solutions. These findings are empirical and should be read
together with the theoretical scope in Section~\ref{sec:theory}: the strict
convergence result applies to the convex-affine backbone, while the optional
stabilization modules are analyzed as perturbations of the backbone feedback
loop.

\section{Conclusion}
\label{sec:conclusion}

This paper addresses the multiplier updating bottleneck in stochastic constrained optimization by proposing a residual control signal that separates instantaneous projected pressure from internal memory updates. Based on this signal, we construct the RCML algorithmic framework, establish its finite-gain convergence and stochastic bounds through theoretical analysis, and empirically validate its effectiveness. However, the proposed methodology has some limitations. First, the stochastic stabilization modules introduce additional hyperparameters, and the constraint filtering mechanism fundamentally trades tracking lag for variance reduction, which may lead to transient constraint breaches under rapid distribution shifts. Second, the algorithm lacks an active mechanism to avoid nonsmooth projection boundaries (kinks); thus, achieving the sharper \(O(1/B)\) noise floor passively relies on the local stability of the active set. Future work could improve upon this by developing hyperparameter-free adaptive residual controllers or incorporating margin-aware mechanisms that actively stabilize the trajectory near nonsmooth boundaries without sacrificing response speed.

\bibliography{VM}
\bibliographystyle{tmlr}
\newpage

\appendix
\section{Hyperparameter Roles and Tuning Guidelines for RCML}
\label{subsec:hyperparameter_guidelines}

Table
\ref{tab:hyperparameter_guidelines} summarizes the practical hyperparameter tuning
guidelines.
Some are standard in stochastic primal--dual
optimization, some belong to the projected-pressure residual mechanism, and some
are only used when optional stabilization modules are activated. 

\begingroup
\footnotesize
\setlength{\tabcolsep}{3.0pt}
\renewcommand{\arraystretch}{1.08}

\begin{longtable}{
		C{0.11\textwidth}
		C{0.16\textwidth}
		C{0.29\textwidth}
		C{0.12\textwidth}
		C{0.25\textwidth}
	}
	\caption{
		Roles and tuning guidelines of the main hyperparameters.
		``Origin'' indicates whether a parameter belongs to the stochastic primal--dual
		interface, the projected-pressure residual mechanism, or an optional RCML
		stabilization module.
	}
	\label{tab:hyperparameter_guidelines}
	\\
	\toprule
	\textbf{Parameter}
	&
	\textbf{Origin}
	&
	\textbf{Role}
	&
	\textbf{Tune?}
	&
	\textbf{Practical guideline}
	\\
	\midrule
	\endfirsthead
	
	\caption[]{Roles and tuning guidelines of the main hyperparameters. Continued.}
	\\
	\toprule
	\textbf{Parameter}
	&
	\textbf{Origin}
	&
	\textbf{Role}
	&
	\textbf{Tune?}
	&
	\textbf{Practical guideline}
	\\
	\midrule
	\endhead
	
	\midrule
	\multicolumn{5}{r}{Continued on next page}
	\\
	\endfoot
	
	\bottomrule
	\endlastfoot
	
	\(\alpha_k\)
	&
	Primal--dual interface
	&
	Primal stepsize for updating \(x_k\). It controls objective descent and the
	sensitivity of the primal step to multiplier pressure.
	&
	Yes
	&
	Tune as in standard stochastic optimization. Use smaller values when constraint
	gradients are noisy or when the projected pressure \(\lambda_k\) is large.
	\\
	\midrule
	
	\(\eta_k\)
	&
	Primal--dual interface
	&
	Multiplier-memory stepsize in \eqref{eq:pd_template}. It controls how strongly
	the feedback signal \(s_k\) changes \(u_k\).
	&
	Yes
	&
	Tune jointly with \(\kappa_{\text{I}}\). For residual-integral variants, the effective
	gain \(\beta_k=\eta_k\kappa_{\text{I}}\) should satisfy \(0\le\beta_k\le1\).
	\\
	\midrule
	
	\(u_0\)
	&
	Primal--dual interface
	&
	Initial multiplier memory.
	&
	Usually no
	&
	Set \(u_0=0\) unless prior knowledge indicates that some constraints should be
	active at initialization.
	\\
	\midrule
	
	\(\boldsymbol\rho_k\)
	&
	Projected-pressure mechanism
	&
	Pressure-scale vector used in \eqref{eq:pressure_residual}. It controls how
	strongly each constraint signal is amplified before projection.
	&
	Yes
	&
	Use either a fixed isotropic scale \(\boldsymbol\rho_k=\rho_0\mathbf 1\) or an
	adaptive coordinatewise scale. Larger values enforce constraints more aggressively
	but may amplify noisy mini-batch estimates.
	\\
	\midrule
	
	\(\rho_0\)
	&
	Fixed pressure scale
	&
	Scalar scale used when the pressure scale is isotropic,
	\(\boldsymbol\rho_k=\rho_0\mathbf 1\).
	&
	Yes
	&
	Tune together with \(\alpha_k\). Larger \(\rho_0\) increases feasibility pressure
	but also increases sensitivity to constraint noise.
	\\
	\midrule
	
	\(\rho_{k,i}\)
	&
	Adaptive pressure scaling
	&
	The \(i\)-th coordinate of the adaptive pressure-scale vector
	\(\boldsymbol\rho_k\).
	&
	Usually no after rule is fixed
	&
	Computed by \eqref{eq:adaptive_scale}. Tune the base scale and clipping bounds
	rather than each coordinate separately.
	\\
	\midrule
	
	\(\kappa_{\text{I}}\)
	&
	Residual memory tracking
	&
	Integral gain from the residual \(d_k\) to the memory signal \(s_k\). It controls
	finite-gain tracking speed.
	&
	Yes
	&
	Core RCML parameter. Small values produce smoother memory; large values approach
	projected ALM replacement. Tune through \(\beta_k=\eta_k\kappa_{\text{I}}\).
	\\
	\midrule
	
	\(\gamma_k\)
	&
	Constraint filtering
	&
	Smoothing gain in \eqref{eq:constraint_filter}.
	&
	Yes
	&
	Use \(\gamma_k=1\) for unfiltered feedback. Decrease it under high mini-batch
	noise, but avoid excessive lag.
	\\
	\midrule
	
	\(v_0\)
	&
	Adaptive scaling state
	&
	Initial second-moment state for constraint magnitudes.
	&
	No
	&
	Set \(v_0=0\). Bias correction in \(\widehat v_{k+1}\) reduces initialization
	effects.
	\\
	\midrule
	
	\(\eta_v\)
	&
	Adaptive pressure scaling
	&
	Second-moment update rate for the adaptive scale.
	&
	Low priority
	&
	Use a small to moderate value. Larger values adapt faster but may make
	\(\boldsymbol\rho_k\) fluctuate.
	\\
	\midrule
	
	\(\kappa_\rho\)
	&
	Adaptive pressure scaling
	&
	Base scale in \eqref{eq:adaptive_scale}.
	&
	Yes for adaptive variants
	&
	Controls the average level of \(\rho_{k,i}\). Tune together with
	\(\rho_{\min}\) and \(\rho_{\max}\).
	\\
	\midrule
	
	\(\rho_{\min},\rho_{\max}\)
	&
	Adaptive pressure scaling
	&
	Lower and upper clipping bounds for coordinatewise pressure scales.
	&
	Yes for adaptive variants
	&
	Use \(\rho_{\min}>0\) to avoid vanishing pressure scaling and \(\rho_{\max}\) to
	limit noise amplification.
	\\
	\midrule
	
	\(\epsilon\)
	&
	Adaptive pressure scaling
	&
	Numerical stabilizer in \(\sqrt{\widehat v_{k+1,i}+\epsilon}\).
	&
	Usually no
	&
	Use a small constant, such as \(10^{-8}\) or \(10^{-6}\), depending on numerical
	precision.
	\\
	\midrule
	
	\(\nu\)
	&
	Residual-\(\nu\)PI correction
	&
	Residual smoothing factor in \eqref{eq:residual_filter}.
	&
	Yes for robust variant
	&
	Larger values produce smoother but slower dynamic correction. Smaller values
	respond faster to changing residual trends.
	\\
	\midrule
	
	\(\kappa_{\text{P}}\)
	&
	Residual-\(\nu\)PI correction
	&
	Gain multiplying \(\xi_k-\xi_{k-1}\) in \eqref{eq:nupi_update}.
	&
	Yes for robust variant
	&
	Start from \(\kappa_{\text{P}}=0\) and increase gradually. Large values may accelerate
	response but can also amplify residual oscillations.
	\\
	\midrule
	
	\(\bar c_{-1},\xi_{-1}\)
	&
	Auxiliary states
	&
	Initial states for constraint filtering and residual smoothing.
	&
	No
	&
	Set to zero unless warm-start information is available.
	\\
	\midrule
	
	\(B_k\)
	&
	Stochastic estimation
	&
	Mini-batch size for estimating \(\hat g_k\), \(\hat c_k\), and
	\(\widehat J_{c,k}\).
	&
	Task-dependent
	&
	Larger batches reduce feedback noise but increase computation. RCML is designed
	for settings where full-batch constraints are expensive.
	\\
	
\end{longtable}
\endgroup

In practice, we tune RCML hierarchically. First, choose the primal stepsize
\(\alpha_k\) and a pressure scale. For the basic projected-pressure update, this
is usually a fixed isotropic scale \(\boldsymbol\rho_k=\rho_0\mathbf 1\).
Second, tune the effective residual tracking gain
\(\beta_k=\eta_k\kappa_{\text{I}}\), which determines how close the method is to full
projected ALM replacement. Third, enable constraint filtering by decreasing
\(\gamma_k\) when mini-batch constraint estimates are visibly noisy. Finally,
activate adaptive pressure scaling when constraints have heterogeneous magnitudes,
and activate residual-\(\nu\)PI correction when multiplier oscillations remain
after filtering.

This hierarchy also clarifies which parameters are intrinsic to RCML. The primal
stepsize \(\alpha_k\), multiplier-memory stepsize \(\eta_k\), initial multiplier
\(u_0\), and batch size \(B_k\) are standard components of stochastic primal--dual
learning. The projected-pressure scale \(\boldsymbol\rho_k\), residual signal
\(d_k\), and tracking gain \(\kappa_{\text{I}}\) define the RCML backbone. The parameters
\(\gamma_k\), \(\eta_v\), \(\kappa_\rho\), \(\rho_{\min}\), \(\rho_{\max}\),
\(\epsilon\), \(\nu\), and \(\kappa_{\text{P}}\) belong to optional stabilization modules.
Therefore, the minimal RCML implementation only requires tuning
\(\alpha_k\), \(\rho_0\), and \(\beta_k=\eta_k\kappa_{\text{I}}\), while the additional
parameters are introduced only when filtering, adaptive scaling, or dynamic
correction is used.

% ============================================================
% Appendix material: place after \appendix in the final manuscript.
% If your manuscript has not yet called \appendix, uncomment the next line.
% \appendix
% ============================================================

\section{Additional Experimental Details}
\label{app:experimental_details}

This appendix summarizes the implementation details needed to reproduce the
experiments. All experiments were conducted on a workstation equipped with an
NVIDIA RTX 4060 Ti GPU and an Intel Core i5-13490F CPU. The fair-ranking
experiment uses PyTorch with CUDA when available, while the other experiments are
implemented using NumPy. The public repository contains one main script for each
experiment.

\subsection{Experiment 1 Details}
\label{app:exp1_details}

Experiment~1 uses synthetic stochastic LP, convex QP, and mildly nonconvex QP
tasks with \(d=30\) decision variables and \(m=30\) inequality constraints. For
each task, we generate \(2048\) objective samples and \(2048\) constraint samples.
The gradient mini-batch size is \(32\), and each method is evaluated over \(10\)
random seeds. The reliability tolerance is \(5\times10^{-2}\).

The stationary regime uses \(500\) iterations, a tail window of \(50\), constraint
batch size \(32\), and no additional observation noise. The high-noise regime
uses \(1500\) iterations, a tail window of \(150\), constraint batch size \(4\),
and additive constraint noise with standard deviation \(0.25\). The
heterogeneous-scale regime uses \(700\) iterations, a tail window of \(70\),
constraint batch size \(16\), additive noise with standard deviation \(0.05\),
and log-uniform positive rescaling of constraint channels. The active-set
switching regime uses \(500\) iterations, a switch at iteration \(250\), a tail
window of \(50\), constraint batch size \(16\), and additive constraint noise with
standard deviation \(0.07\). Before the switch, constraints are tightened to
accumulate multiplier memory; after the switch, they are relaxed to test
stale-memory release. The optional residual-\(\nu\)PI and \(\tau\)-gated variants
are used only in the active-set switching ablation and are not treated as default
RCML methods.

\subsection{Experiment 2 Details}
\label{app:exp2_details}

Experiment~2 uses a tolerance-aware stochastic energy-reserve allocation problem
with \(d=20\) decision variables and \(m=30\) constraints. The training and
testing sample sizes are \(30000\) and \(60000\), respectively. Each method is
evaluated over \(5\) random seeds for \(1600\) iterations. Tail metrics are
computed over the last \(300\) iterations, and evaluation is performed every
\(10\) iterations.

The normal-noise regime uses batch size \(B=256\), lognormal noise scale \(0.25\),
and tolerance \(\delta_{\rm tol}=10^{-2}\). The moderate-noise stress regime uses
batch size \(B=96\), lognormal noise scale \(0.30\), spike probability \(0.05\),
spike scale \(2.2\), spike fraction \(0.20\), and tolerance
\(\delta_{\rm tol}=2\times10^{-2}\). The tolerance-aware score uses
\(\omega_{\rm tol}=100\).

The common optimization parameters are primal learning rate \(5\times10^{-2}\),
multiplier-memory learning rate \(\eta_u=4\times10^{-2}\), fixed isotropic
pressure scale \(\rho_0=1.0\), residual smoothing factor \(\nu=0.65\), integral
gain \(\kappa_{\text{I}}=1.0\), and residual-\(\nu\)PI gain \(\kappa_{\text{P}}=0.05\). For the
selected \textsc{RCML-Adaptive} configuration, we use \(\gamma=0.70\),
\(\eta_v=0.05\), \(\kappa_\rho=0.80\), \(\rho_{\min}=0.15\), and
\(\rho_{\max}=4.0\). The configuration sweep in
Table~\ref{tab:exp2_config_selection} varies \(\gamma\), \(\kappa_\rho\),
\(\rho_{\min}\), and the multiplier-memory learning rate \(\eta_u\).

\subsection{Experiment 3 Details}
\label{app:exp3_details}

Experiment~3 uses a nonconvex stochastic pricing-inventory allocation problem
with \(d=30\) items. For item \(i\), the decision variables are inventory \(x_i\)
and price \(r_i\). Demand follows an exponential price-response model with
stochastic shocks,
\[
D_i(r_i,\xi_i)
=
\bar D_i \exp[-a_i(r_i-r_i^0)]\xi_i + \mathrm{spike}_i .
\]
The objective is to maximize expected profit, equivalently to minimize negative
profit. The constraint requires expected shortage to remain below an
item-dependent tolerance.

The training and testing sample sizes are \(20000\) and \(25000\), respectively.
Each method is evaluated with \(5\) random seeds and \(5\) random initializations,
resulting in \(25\) seed--initialization runs per scenario. Each run uses \(650\)
iterations, with tail metrics computed over the last \(120\) iterations. A run is
considered reliable if its tail tolerance-feasible rate is at least \(0.9\).

The normal-noise regime uses batch size \(B=256\) and tolerance
\(\delta_{\rm tol}=1.5\times10^{-2}\). The moderate-noise stress regime uses
batch size \(B=128\) and tolerance \(\delta_{\rm tol}=3.0\times10^{-2}\). The
primal learning rates are \(4.0\times10^{-2}\) for inventory and
\(2.5\times10^{-2}\) for price. The multiplier-memory learning rate is
\(\eta_u=4.0\times10^{-2}\), and the fixed isotropic pressure scale is
\(\rho_0=10.0\). For \textsc{RCML-Adaptive}, we use \(\gamma=0.70\),
\(\eta_v=0.05\), \(\kappa_\rho=0.80\), \(\rho_{\min}=0.15\), and
\(\rho_{\max}=10.0\).

\subsection{Experiment 4 Details}
\label{app:exp4_details}

Experiment~4 uses a large-scale learning-to-rank dataset. The training
set contains \(10000\) queries and the test set contains \(3000\) queries. Each query contains \(30\) candidate items with \(24\)-dimensional features and a binary group attribute. The neural ranker is a two-hidden-layer multilayer perceptron with hidden dimension \(64\) and ReLU activations. The model is trained with AdamW for \(1000\) iterations using batch size \(128\), learning rate \(10^{-3}\), weight decay \(10^{-5}\), and gradient clipping threshold \(5.0\). Each method is evaluated over \(5\) random seeds.

The ranking loss is a ListNet-type loss. Soft exposure is computed using softmax temperature \(0.75\), while the target distribution uses temperature \(1.0\). The evaluation fairness threshold is \(\epsilon_{\rm eval}=0.012\), and the tolerance band is \(\delta_{\rm tol}=0.003\), giving the admissible exposure-gap threshold \(0.015\). For residual-controlled methods, the internal training threshold is
\(\epsilon_{\rm train}=0.010\). The reported utility metric is NDCG@10.

The multiplier-control parameters are multiplier-memory learning rate
\(\eta_u=0.16\), fixed isotropic pressure scale \(\rho_0=18.0\), projected-ALM
fixed isotropic pressure scale \(\rho_{0,\rm ALM}=10.0\), filtering parameter
\(\gamma=0.85\), adaptive scaling rate \(\eta_v=0.05\), adaptive base scale
\(\kappa_\rho=4.0\), clipping bounds \(\rho_{\min}=2.0\), \(\rho_{\max}=60.0\),
residual smoothing factor \(\nu=0.70\), integral gain \(\kappa_{\text{I}}=1.50\), and
residual-\(\nu\)PI gain \(\kappa_{\text{P}}=0.30\). The fixed fairness penalty baseline
uses penalty weight \(300.0\), and the stochastic barrier baseline uses barrier
parameter \(0.020\), barrier shift \(0.003\), and numerical barrier constant
\(10^{-5}\). The synthetic ranking data are generated with group-dependent
relevance bias, group-dependent feature shift, query-level group skew, and
additive noise, so that the unconstrained ERM solution violates the admissible
exposure band and activates multiplier-based constrained training.

\subsection{Additional Ablation Tables}
\label{app:additional_ablation_tables}

\begin{table}[H]
\centering
\caption{
Configuration selection for \textsc{RCML-Adaptive}.
The screening score is averaged over the two stochastic feedback regimes; lower is better.
}
\label{tab:exp2_config_selection}
\footnotesize
\setlength{\tabcolsep}{\exptabcolsep}
\renewcommand{\arraystretch}{\exptabstretch}
\begin{adjustbox}{max width=\textwidth,center}
\begin{tabular}{ccccccccc}
\toprule
Variant
& \(\gamma\)
& \(\kappa_\rho\)
& \(\rho_{\min}\)
& \(\eta_u\)
& Score
& TolFeas
& TA
& Cost \\
\midrule
\textsc{RCML-Adaptive-lowlag}
& \(0.70\)
& \(0.80\)
& \(0.15\)
& \(0.04\)
& \(\mathbf{6.098}\)
& \(\mathbf{0.743}\)
& \(\mathbf{5.673}\)
& \(5.569\) \\
\textsc{RCML-Adaptive-strong}
& \(0.50\)
& \(1.00\)
& \(0.15\)
& \(0.04\)
& \(6.298\)
& \(0.707\)
& \(5.707\)
& \(5.567\) \\
\textsc{RCML-Adaptive-mild}
& \(0.50\)
& \(0.80\)
& \(0.10\)
& \(0.04\)
& \(6.298\)
& \(0.707\)
& \(5.707\)
& \(5.567\) \\
\textsc{RCML-Adaptive-fastdual}
& \(0.60\)
& \(0.80\)
& \(0.20\)
& \(0.05\)
& \(6.511\)
& \(0.683\)
& \(5.719\)
& \(5.568\) \\
\textsc{RCML-Adaptive-default}
& \(0.35\)
& \(0.65\)
& \(0.10\)
& \(0.04\)
& \(6.561\)
& \(0.677\)
& \(5.736\)
& \(5.566\) \\
\bottomrule
\end{tabular}
\end{adjustbox}
\end{table}

\begin{table}[H]
\centering
\caption{
Optional dynamic extensions under active-set switching.
Constraints are tightened before the switch and relaxed after the switch.
PostDualTV is the post-switch average of \(\|u_{k+1}-u_k\|_2\), and ExcessMem is the post-switch average of \(\|u_k\|_2/(1+\|u_{k_s}\|_2)\).
}
\label{tab:exp1_optional_extensions}
\footnotesize
\setlength{\tabcolsep}{\exptabcolsep}
\renewcommand{\arraystretch}{\exptabstretch}
\begin{adjustbox}{max width=\textwidth,center}
\begin{tabular}{ccccccc}
\toprule
Problem & Method & PostObj & PostViol & PostDualTV & ExcessMem & Runtime \\
\midrule
LP & \textsc{RCML-Adaptive}
& \(-1.061{\times}10^{1}\)
& \(2.430{\times}10^{-2}\)
& \(4.232{\times}10^{-3}\)
& \(2.785{\times}10^{-1}\)
& \(1.09{\times}10^{-1}\) \\
LP & \textsc{RCML-Adaptive-}\(\tau\)
& \(-1.049{\times}10^{1}\)
& \(2.437{\times}10^{-2}\)
& \(4.104{\times}10^{-3}\)
& \(2.661{\times}10^{-1}\)
& \(1.04{\times}10^{-1}\) \\
LP & \textsc{Residual-}\(\nu\)\textsc{PI}
& \(-1.062{\times}10^{1}\)
& \(2.423{\times}10^{-2}\)
& \(5.099{\times}10^{-3}\)
& \(2.786{\times}10^{-1}\)
& \(1.10{\times}10^{-1}\) \\
\midrule
QP & \textsc{RCML-Adaptive}
& \(5.662{\times}10^{-1}\)
& \(3.745{\times}10^{-3}\)
& \(1.793{\times}10^{-3}\)
& \(9.574{\times}10^{-2}\)
& \(1.18{\times}10^{-1}\) \\
QP & \textsc{RCML-Adaptive-}\(\tau\)
& \(6.238{\times}10^{-1}\)
& \(3.236{\times}10^{-3}\)
& \(1.684{\times}10^{-3}\)
& \(9.237{\times}10^{-2}\)
& \(1.17{\times}10^{-1}\) \\
QP & \textsc{Residual-}\(\nu\)\textsc{PI}
& \(5.664{\times}10^{-1}\)
& \(3.731{\times}10^{-3}\)
& \(2.090{\times}10^{-3}\)
& \(9.570{\times}10^{-2}\)
& \(1.18{\times}10^{-1}\) \\
\midrule
NCVQP & \textsc{RCML-Adaptive}
& \(-1.110{\times}10^{-1}\)
& \(0.000\)
& \(5.149{\times}10^{-4}\)
& \(6.328{\times}10^{-2}\)
& \(8.53{\times}10^{-2}\) \\
NCVQP & \textsc{RCML-Adaptive-}\(\tau\)
& \(-1.454{\times}10^{-1}\)
& \(0.000\)
& \(5.244{\times}10^{-4}\)
& \(6.428{\times}10^{-2}\)
& \(8.39{\times}10^{-2}\) \\
NCVQP & \textsc{Residual-}\(\nu\)\textsc{PI}
& \(-1.283{\times}10^{-1}\)
& \(0.000\)
& \(5.202{\times}10^{-4}\)
& \(6.362{\times}10^{-2}\)
& \(8.73{\times}10^{-2}\) \\
\bottomrule
\end{tabular}
\end{adjustbox}
\end{table}

\end{document}